\documentclass{article}
\usepackage[utf8]{inputenc}

\usepackage{amsmath}
\usepackage{amssymb}
\usepackage{amsthm}
\usepackage{subcaption}
\usepackage{enumitem}
\usepackage{xcolor}
\usepackage{algorithm}
\usepackage{graphicx}
\usepackage{placeins}
\usepackage{hyperref}

\makeatletter
\newcommand{\printfnsymbol}[1]{%
  \textsuperscript{\@fnsymbol{#1}}%
}
\makeatother

\usepackage{amsmath,amsfonts,bm}

\def\eqref#1{equation~\ref{#1}}
\def\Eqref#1{Equation~\ref{#1}}
\def\1{\bm{1}}

\def\rva{{\mathbf{a}}}

\def\rve{{\mathbf{e}}}

\def\rvg{{\mathbf{g}}}

\def\rvu{{\mathbf{i}}}

\def\rvu{{\mathbf{u}}}

\def\rvw{{\mathbf{w}}}
\def\rvx{{\mathbf{x}}}
\def\rvy{{\mathbf{y}}}
\def\rvz{{\mathbf{z}}}

\def\hrvx{{\hat{\rvx}}}
\def\hrvz{{\hat{\rvz}}}

\def\trve{{\tilde{\rve}}}
\def\trvw{{\tilde{\rvw}}}

\def\rvzero{{\mathbf{0}}}
\def\rvone{{\mathbf{1}}}

\def\rmA{{\mathbf{A}}}

\def\rmD{{\mathbf{D}}}

\def\rmI{{\mathbf{I}}}

\def\rmQ{{\mathbf{Q}}}

\def\rmW{{\mathbf{W}}}

\def\trmW{{\tilde{\rmW}}}

\DeclareMathAlphabet{\mathsfit}{\encodingdefault}{\sfdefault}{m}{sl}
\SetMathAlphabet{\mathsfit}{bold}{\encodingdefault}{\sfdefault}{bx}{n}

\def\sN{{\mathbb{N}}}

\def\sR{{\mathbb{R}}}

\newcommand{\E}{\mathbb{E}}

\newcommand{\R}{\mathbb{R}}

\DeclareMathOperator*{\argmin}{arg\,min}
\DeclareMathOperator*{\relu}{ReLU}

\newcommand{\st}{\operatorname{s.t.}}

\DeclareMathOperator{\trace}{trace}

\newcommand{\norm}[1]{\left\lVert#1\right\rVert}

\newcommand{\abs}[1]{\left\lvert#1\right\rvert}

\newtheorem{theorem}{Theorem}
\newtheorem{lemma}{Lemma}
\newtheorem{definition}{Definition}
\newtheorem{corollary}{Corollary}

\newcommand{\Support}{{\textnormal{Support}}}

\newcommand{\thickbar}[1]{\mathbf{\bar{\text{$#1$}}}}

\def\N{{\mathcal{N}}}
\def\S{{\mathcal{S}}}
\def\hS{{\hat{\S}}}
\DeclareMathOperator*{\spark}{spark}

\DeclareMathOperator{\rank}{rank}
\def\brmW{{\thickbar \rmW}}
\def\lmax{{\lambda_{\max}}}
\def\lmin{{\lambda_{\min}}}

\DeclareMathOperator{\subspark}{sub-spark}
\DeclareMathOperator{\subrank}{sub-rank}

\title{When and How Can Deep Generative Models be Inverted?}
\author{%
  Aviad Aberdam\thanks{These authors are listed in alphabetical order.} \\
  Electrical Engineering \\
  Technion, Israel\\
  {\tt\small aaberdam@cs.technion.ac.il}
  \and
  Dror Simon\printfnsymbol{1} \\
  Computer Science \\
  Technion, Israel\\
  {\tt\small dror.simon@cs.technion.ac.il}
  \and
  Michael Elad \\
  Computer Science \\
  Technion, Israel\\
  {\tt\small elad@cs.technion.ac.il} \\
}
\date{}

\begin{document}

\maketitle

\begin{abstract}

    Deep generative models (e.g. GANs and VAEs) have been developed quite extensively in recent years. Lately, there has been an increased interest in the inversion of such a model, i.e. given a (possibly corrupted) signal, we wish to recover the latent vector that generated it. % the (clean) signal when passed through the model. 
    Building upon sparse representation theory, we define conditions that are applicable to any inversion algorithm (gradient descent, deep encoder, etc.), under which such generative models are invertible with a unique solution. Importantly, the proposed analysis is applicable to any trained model, and does not depend on Gaussian i.i.d. weights. Furthermore, we introduce two layer-wise inversion pursuit algorithms
    % Latent-Pursuit, a layer-wise inversion algorithm 
    for trained % non-random 
    generative networks of arbitrary depth, and accompany these with recovery guarantees. Finally, we validate our theoretical results numerically and show that our method outperforms gradient descent when inverting such generators, both for clean and corrupted signals.
    
    %Building upon sparse representation theory, we define conditions that are applicable to any algorithm (gradient descent, deep encoder, etc.), under which such generative model is invertible and has a unique solution. Then, we introduce the Latent-Pursuit, a layer-wise inversion algorithm, and suggest novel recovery guarantees for trained non-random generative networks of arbitrary depth for the noiseless and the noisy cases. Finally, we validate our theoretical results numerically and show that our method outperforms gradient descent when inverting both clean and corrupted signals.
    
    % In recent years, deep generative models have been developed extensively. Lately, there has been an increased interest in the inversion of such a model, i.e. given a (corrupted) signal, we wish to recover the latent vector that generated the (clean) signal when passed through the model. Using sparse representation theory, we suggest novel recovery guarantees for $\relu$ activated generative models of arbitrary depth in the realizable and the non-realizable cases. In contrast to other recent works, our analysis does not require an expansion in the weights of the network, nor does it depend on Gaussian i.i.d. weights. Furthermore, we present a novel provably converging algorithm that outperforms inversion by back-propagation in terms of reconstruction error.
\end{abstract}

\begin{figure}[t]
    \centering
    \begin{subfigure}{0.32\textwidth}
    \centering
        \includegraphics[width=1\linewidth]{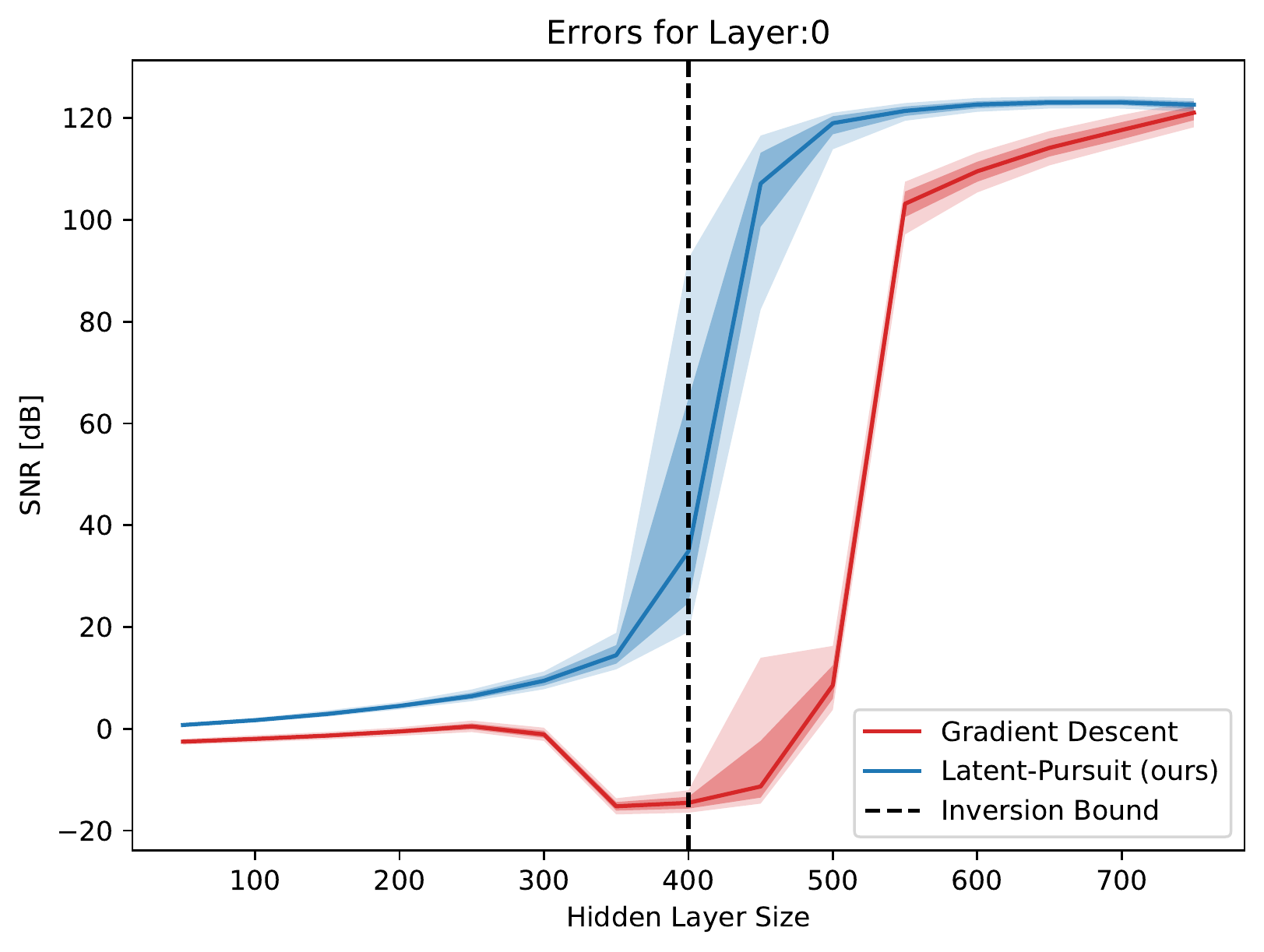}
        \caption{Latent vector recovery}
    \end{subfigure}
    \begin{subfigure}{0.32\textwidth}
    \centering
        \includegraphics[width=1\linewidth]{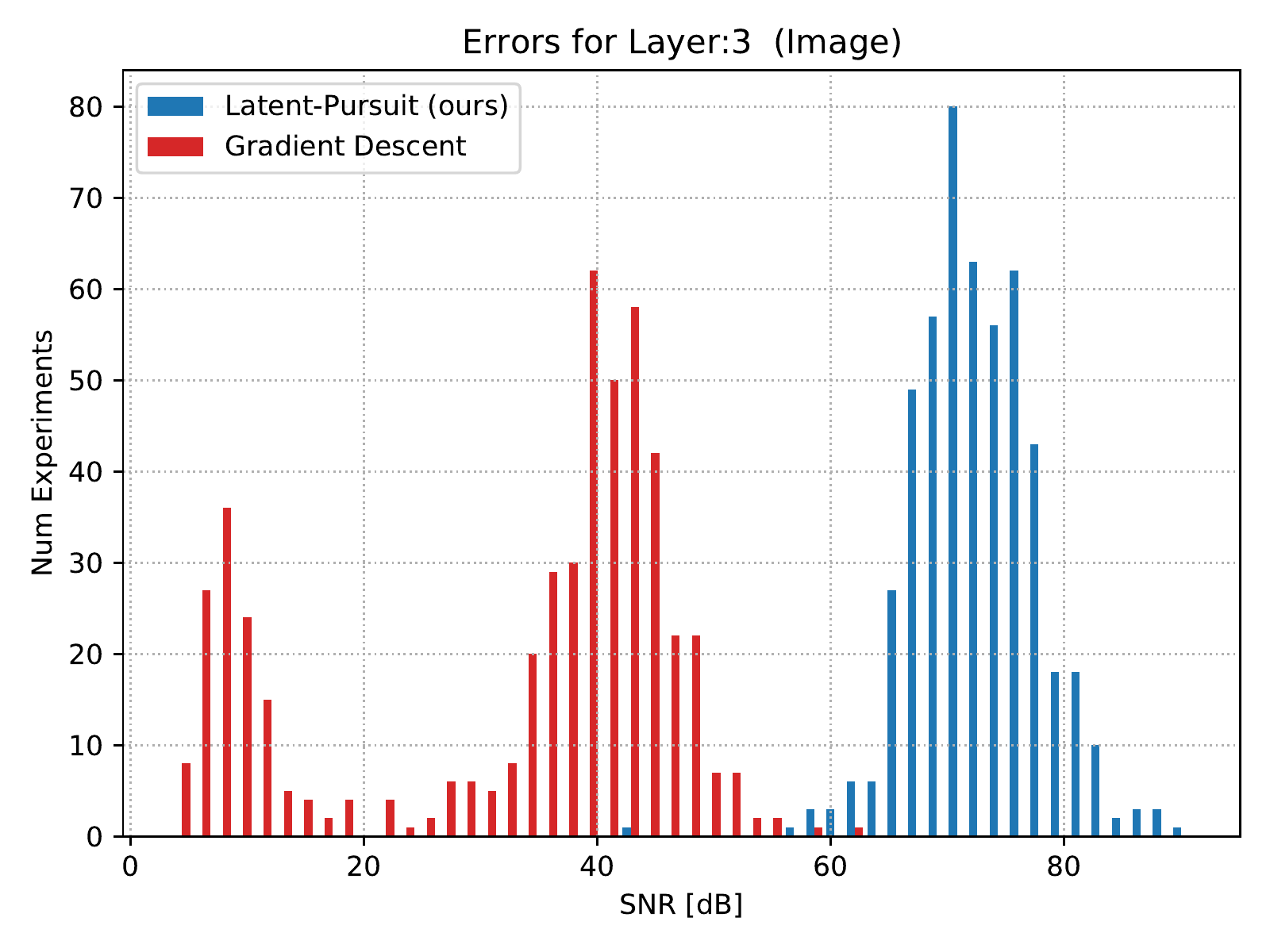}
        \caption{Image reconstruction}
    \end{subfigure}
    \begin{subfigure}{0.25\textwidth}
    \centering
        \includegraphics[trim={195 80 85 88},clip,width=1\linewidth]{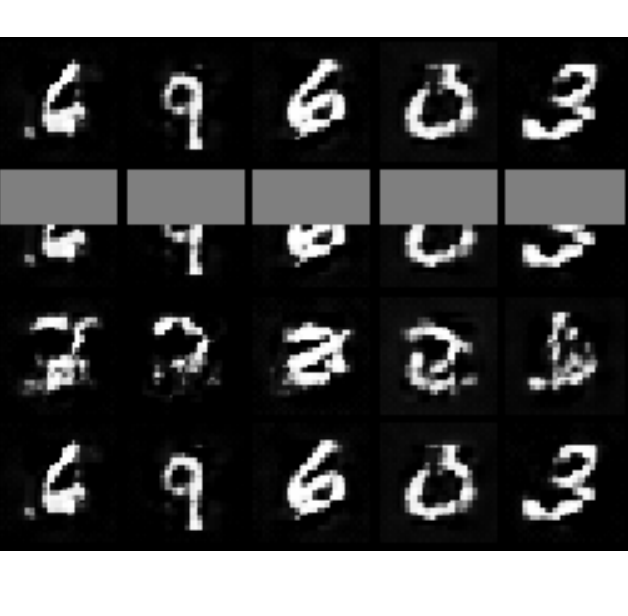}
        \caption{Image inpainting}
    \end{subfigure}
    \caption{In this paper we introduce conditions for a generative model to be invertible regardless of the inversion algorithm. As depicted in (a), in GANs with random weights, a too small hidden layer results in a non-invertible generative process. We propose the Latent-Pursuit algorithm, which empirically outperforms the gradient descent method on trained GANs in the tasks of clean images reconstruction (b) and image inpainting (c) (from top to bottom: clean image, corrupted image, gradient descent and our results).}
    \label{fig:my_label}
\end{figure}

\section{Introduction}

% Generative models are mathematical functions that aim at learning the distribution of a given dataset. In the field of image processing, many such models, e.g. total-variation, field-of-experts and sparse priors \cite{rudin1992nonlinear,roth2005fields,bruckstein2009sparse}, were proposed throughout the years. While the generative power of these models are debatable, these models were mainly used as priors for the sake of reconstructing corrupted signals, accompanied by various algorithms that posses theoretical reconstruction guarantees. 
In the past several years, deep generative models, e.g. Generative Adversarial Networks (GANs) \cite{goodfellow2014generative} and Variational Auto-Encoders (VAEs) \cite{kingma2013auto}, have been greatly developed, leading to networks that can generate images, videos, and speech voices among others, that look and sound authentic to humans. Loosely speaking, these models learn a mapping from a random low-dimensional latent space to the training data distribution, obtained in an unsupervised manner.

%In this work we focus on inverting such a generative neural network by restoring the generative latent space from a given input signal. This inversion task has gained an increased interest recently as it is the basis of various applications, including visual manipulation, compressed sensing, image interpolation, image generation, and others \cite{zhu2016generative,bora2017compressed,simon2020barycenters,zhu2016generative}.

Interestingly, deep generative models are not used only to generate arbitrary signals. Recent work rely on the inversion of these models to perform visual manipulation, compressed sensing, image interpolation, image generation, and others \cite{zhu2016generative,bora2017compressed,simon2020barycenters,zhu2016generative}. In this work, we study this inversion task.
Formally, denoting the signal to invert by $\rvy\in\sR^n$, the generative model as $G:\sR^{n_0}\to\sR^n$, and the latent vector as $\rvz\in\sR^{n_0}$, we study the following
inversion problem:
\begin{equation}
    \rvz^* = \argmin_{\rvz} \frac{1}{2} \|G(\rvz)-\rvy\|_2^2,
    \label{eq:inverse_problem}
\end{equation}
where $G$ is assumed to be a feed-forward neural network.
% or in other words, to \emph{invert} the model $G$ which is assumed to be a feed-forward neural network.
% Intuitively, we expect deep generative models to outperform traditional approaches due to their success in signal generation and indeed, recent papers show the benefit of inverting these models for various applications such as compressed sensing, image interpolation, image generation, and others \cite{bora2017compressed,simon2020barycenters,zhu2016generative}.

The first question that comes to mind is whether this model is invertible, or equivalently, does \Eqref{eq:inverse_problem} have a unique solution? In this work, we establish theoretical conditions that guarantee the invertibility of the model $G$. Notably, the provided theorems are applicable to general non-random generative models, and do not depend on the chosen inversion algorithm.

Once the existence of a unique solution is established, the next challenge is to provide a recovery algorithm that is guaranteed to obtain the sought solution. A common and simple approach is to draw a random vector $\rvz$ and iteratively update it using gradient descent, opting to minimize \Eqref{eq:inverse_problem} \cite{zhu2016generative,bora2017compressed}. Unfortunately, this approach has theoretical guarantees only in limited scenarios \cite{hand2018phase,hand2019global}, since the inversion problem is generally non-convex. An alternative approach is to train an encoding neural network that maps images to their latent vectors \cite{zhu2016generative,donahue2016adversarial,bau2019seeing,simon2020barycenters}; however, this method is not accompanied by any theoretical justification.

We adopt a third approach in which the generative model is inverted in an analytical fashion. Specifically, we perform the inversion layer-by-layer, similar to \cite{lei2019inverting}. Our approach is based on the observation that every hidden layer is an outcome of a weight matrix multiplying a sparse vector, followed by a $\relu$ activation. By utilizing sparse representation theory, the proposed algorithm ensures perfect recovery in the noiseless case and bounded estimation error in the noisy one. Moreover, we show numerically that our algorithm outperforms gradient descent in several tasks, including reconstruction of noiseless and corrupted images.
% An example for this method is given in \cite{lei2019inverting}, where the authors propose to invert each layer by solving a linear programming problem. Unfortunately, the authors of \cite{lei2019inverting} lean on the unrealistic assumption that the weights of the network are drawn at random, among other restrictive assumptions.

% In this work, we see deep generative models as a natural extension to the well-known sparse
% representation prior \cite{bruckstein2009sparse,elad2010sparse}. Specifically, using sparse coding theory, we provide answers to the following questions:
% \begin{itemize}
%     \item Can such (non-random) generative models be inverted? In other words, is the solution for this problem unique?
%     \item Is there an algorithm that guarantees the recovery of the unique solution?
%     \item What is the effect of noise in the signal? Can the reconstruction error be bounded?
% \end{itemize}

\paragraph{Main contributions:} The contributions of this work are both theoretical and practical. We derive theoretical conditions for the invertiblity of deep generative models by ensuring a unique solution for the inversion problem defined in \Eqref{eq:inverse_problem}. Then, by leveraging the inherent sparsity of the hidden layers, we introduce a layerwise inversion algorithm with provable guarantees in the noiseless and noisy settings for trained fully-connected generators. To the best of our knowledge, this is the first work that provides such guarantees for general (non-random) models, addressing both the conceptual inversion and provable algorithms for solving \Eqref{eq:inverse_problem}. Finally, we provide numerical experiments, demonstrating the superiority of our approach over gradient descent in various scenarios.

\subsection{Related Work}

\paragraph{Inverting deep generative models:} A tempting approach for solving \Eqref{eq:inverse_problem} is to use first order methods such as gradient descent. Even though this inversion is generally non-convex, the works in \cite{hand2019global,hand2018phase} show that if the weights are random then, under additional assumptions, no spurious stationary points exist, and thus gradient descent converges to the optimal solution. A different analysis, given in \cite{latorre2019fast}, studies the case of strongly smooth generative models that are near isometry. In this work, we study the inversion of general (non-random and non-smooth) $\relu$ activated generative networks, and provide a provable algorithm that empirically outperforms gradient descent.
A close but different line of theoretical work analyze the compressive sensing abilities of trained deep generative networks \cite{shah2018solving,bora2017compressed}. That said, these works assume that an ideal inversion algorithm, solving \Eqref{eq:inverse_problem}, exists.
% provide theoretical guarantees to recover the original latent code for a variety of inverse problems \cite{hand2019global,huang2018provably, heckel2018deep,hand2018phase}. Unfortunately, these papers rely on impractical assumptions such as random Gaussian weights in their analysis. Another interesting work examines the compressive sensing abilities of trained deep generative networks \cite{bora2017compressed}. While this work shows empirical success, it does not provide a provable algorithm to recover the sought minimzer. 
% In this work, we provide a provable algorithm for inverting a general (nonrandom and nonsmooth) generative network and empirically show its superiority over gradient descent.

\paragraph{Layered-wise inversion:} The closest work to ours, and indeed its source of inspiration, is \cite{lei2019inverting}, where the authors proposed a novel scheme for inverting generative models. By assuming that the input signal was corrupted by bounded noise in terms of $\ell_1$ or $\ell_{\infty}$, they suggest inverting the model using linear programs layer-by-layer. That said, to assure a stable inversion, their analysis is restricted to cases where: (i) the weights of the network are Gaussian i.i.d. variables; (ii) the layers expand such that the number of non-zero elements in each layer is larger than the size of the entire layer preceding it; and (iii) that the last activation function is either $\relu$ or leaky-$\relu$. Unfortunately, as the authors mention in their work, these three assumptions often do not hold in practice. In this work, we do not rely on the distribution of the weights nor on the chosen activation function of the last layer. Furthermore, we relax the expansion assumption and rely only on the expansion of the number of non-zero elements instead.

\paragraph{Neural networks and sparse representation:} In the grand search for a profound theoretical understanding for deep learning, a series of papers suggested a connection between neural networks and sparse coding \cite{papyan2017convolutional,sulam2018multilayer,chun2019convolutional,sulam2019multi,romano2019adversarial,xin2016maximal}. In short, this line of work suggests that the forward pass of a neural network is in fact a pursuit for a multilayer sparse representation. In this work, we expand this proposition by showing that the inversion of a generative neural network is based on sequential sparse coding steps.
% , in contrast to other studies that emphasize their differences \cite{bora2017compressed,shah2018solving, hand2018phase, aubin2019exact, huang2018provably}.

\section{The Generative Model}
\label{sec:gan_model}

\subsection{Notations} 
\label{sec:notations}
We use bold uppercase letters to represent matrices, and bold lowercase letters to represent vectors. The vector $\rvw_{j}$ represents the $j$th column in the matrix $\rmW$. Similarly, the vector $\rvw_{i,j}$ represents the $j$th column in the matrix $\rmW_i$. The activation function $\relu$ is the entry-wise operator $\relu(\rvu) = \max \{\rvu, \rvzero\}$, and $\circ$ denotes an Hadamard product. We denote by $\spark(\rmW)$ the smallest number of columns in $\rmW$ that are linearly-dependent, and by $\norm{\rvx}_0$ the number of non-zero elements in $\rvx$.
The mutual coherence of a matrix $\rmW$ is defined as:
\begin{equation}
    \mu(\rmW) = \max_{i\ne j} \frac{\abs{\rvw_i^T \rvw_j}}{\norm{\rvw_i}_2 \norm{\rvw_j}_2}.
\end{equation}
Finally, we define $\rvx^{\S}$ and $\rmW_i^{\S}$ as the supported vector and the row-supported matrix according to the set $\S$.

\subsection{Problem Statement} 
\label{sec:statement}

\begin{figure}[t]
    \centering
    \includegraphics[width=1\linewidth]{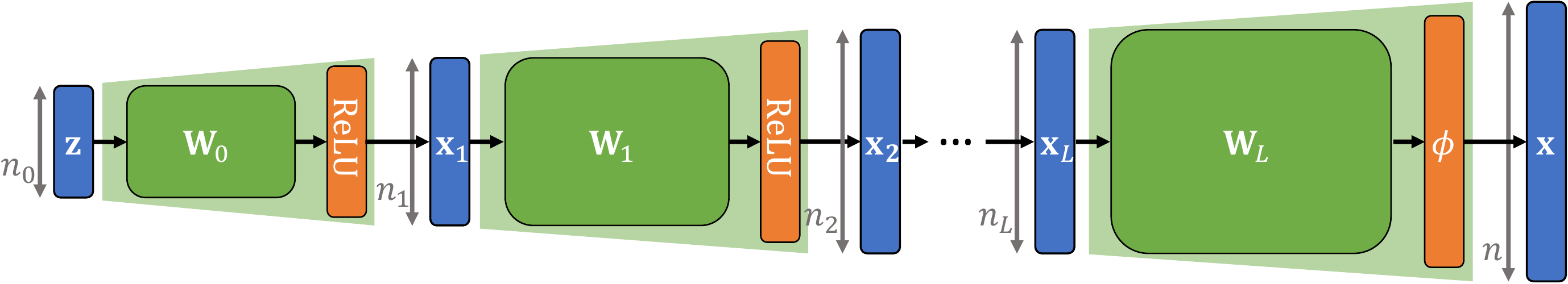}
    \caption{The studied generative model $\rvx = G(\rvz)$.}
    \label{fig:generative_model}
\end{figure}

As depicted in Figure \ref{fig:generative_model}, throughout this work we consider a typical generative scheme $G: \R^{n_0} \rightarrow \R^n$ of the following form:
\begin{equation}\label{eq:generative_model_definition}
\begin{split}
    \rvx_1 & = \relu(\rmW_{0} \rvz), \\
    \rvx_{i+1} & = \relu(\rmW_{i} \rvx_{i}), \text{ for all } i \in \{1, \ldots, L-1\}, \\
    G(\rvz) & = \phi(\rmW_L \rvx_L), 
\end{split}
\end{equation}
where $\rvx_i \in \R^{n_i}$ are the hidden layers, $\rmW_i \in \R^{n_{i+1}\times n_i}$ are the weight matrices, $\rvz \in \R^{n_0}$ is the latent vector that is usually randomly selected from a normal distribution, $\rvz \sim \N(\rvzero, \sigma^2 \rmI_{n_0})$, and $\phi$ is an invertible activation function, e.g. $\tanh$,  sigmoid, or piece-wise linear.

Given a sample $\rvx = G(\rvz)$, that was created by the generative model above, we aim to recover its latent vector $\rvz$. Note that each hidden vector in the model is produced by a $\relu$ activation, leading to hidden layers that are inherently sparse.
This observation supports our approach to study this model utilizing sparse representation theory. In what follows, we use this observation to derive theoretical statements on the invertibility and the stability of this problem, and to develop pursuit algorithms that are guaranteed to restore the original latent vector.

\section{Invertibility and Uniqueness}
\label{sec:uniqueness}

We start by addressing this question: ``Is this generative process invertible?''. In other words, when given a signal that was generated by the model, $\rvx = G(\rvz^*)$, we know that a solution $\rvz^*$ to the inverse problem  exists; however, can we ensure that this is the \emph{only} one?
Theorem \ref{thm:uniqueness} below (its proof is given in Appendix \ref{app:uniqueness}) provides such guarantees, which are based on the sparsity level of the hidden layers and the spark of the weight matrices (see Section \ref{sec:notations}). Importantly, this theorem is not restricted to a specific pursuit algorithm; it can rather be used for any restoration method (gradient descent, deep encoder, etc.) to determine whether the recovered latent vector is the unique solution.

\begin{definition}[$\subspark$]
    Define the $s$-$\subspark$ of a matrix $\rmW$ as the minimal spark of any subset $\S$ of rows of cardinality $|\S| = s$:
    \begin{equation}
        \subspark(\rmW, s) = \min_{|\S|=s} \spark(\rmW^{\S}).
    \end{equation}
\end{definition}

\begin{definition}[$\subrank$]
    Define the $s$-$\subrank$ of a matrix $\rmW$ as the minimal rank over any subset $\S$ of rows of cardinality $|\S| = s$:
    \begin{equation}
        \subrank(\rmW, s) = \min_{|\S|=s} \rank(\rmW^{\S}).
    \end{equation}
\end{definition}

\begin{theorem}[Uniqueness] \label{thm:uniqueness}
    Consider the generative scheme described in \Eqref{eq:generative_model_definition} and a signal $\rvx = G(\rvz^*)$ with a corresponding set of representations $\{\rvx_i^*\}_{i=1}^L$ that satisfy:
    %the inverse problem has a unique solution if the following holds:
    \begin{enumerate}[label=(\roman*)]
        \item $s_L = \norm{\rvx_L^*}_0 < \frac{\spark(\rmW_L)}{2}$.
        \item $s_i = \norm{\rvx_i^*}_0 < \frac{\subspark(\rmW_i, s_{i+1})}{2}$, for all $i \in \{1, \ldots, L-1\}$.
        % \item $n_0 \leq \subrank(\rmW_0, s_1)$.
        \item $n_0 = \subrank(\rmW_0, s_1) \le s_1$.
    \end{enumerate}
    Then, $\rvz^*$ is the unique solution to the inverse problem that meets these sparsity conditions.
\end{theorem}

Theorem \ref{thm:uniqueness} is the first of its kind to provide uniqueness guarantees for general non-statistical weight matrices. Moreover, it only requires an expansion of the layer cardinalities as opposed to \cite{huang2018provably, hand2019global} and \cite{lei2019inverting} that require dimensionality expansion that often does not hold for the last layer (typically $n < n_L$).

% We now emphasize the advantages of Theorem \ref{thm:uniqueness} compared to previous work.
% To invert the generative model, the analysis described in \cite{lei2019inverting} relied on the cardinality of every layer being larger than the dimension of the previous layer, i.e. $n\geq s_L$ and $s_{i} = \norm{\rvx_{i}}_0 \geq n_{i-1}$ for all $i \in \{1,\ldots,L\}$. When this statement holds, the equality constraints $\rvx_{i}^{\S} = \rmW_{i-1}^{\S}\rvx_{i-1}$  are sufficient to guarantee a unique solution to the inversion problem. Comparing this approach to Theorem \ref{thm:uniqueness}, the former suffers from the following disadvantages:
% \begin{itemize}
%     \item The image dimension $n$ is often smaller than the dimension of $\rvx_L$. As declared in \cite{lei2019inverting}, such cases cannot be approach by solving linear equations. However, the proposed approach can treat such a case, as it only requires a small number of non-zero elements.
%     \item Exploiting the inherent sparsity property of the midlayers yields better signal reconstruction in the noisy setting, since we project project the signal onto a lower dimensional subspace.
%     \item In cases where the representation vectors $\rvx_i$ are sparse, the analysis given in \cite{lei2019inverting} requires a large expansion between the layers to guarantee a unique inversion.
%     \item They require large coefficients, random weights, and solving linear programming problems in high dimensions.
%     \item cannot handle non-linear activation in the output.
% \end{itemize}

In the following corollary, we demonstrate the above theorem for the case of random matrices. We leverage the fact that if $\rmW \in \R^{n\times m}$ ($n\le m$), then the probability of heaving $n$ columns that are linearly dependent is essentially zero \cite[Chapter 2]{elad2010sparse}. In fact, since singular square matrices have Lebesgue measure zero, this corollary holds for almost all set of matrices.

\begin{corollary}[Uniqueness for Random Weight Matrices]
    Assume that the weight matrices comprise of random independent and identically distributed entries (say Gaussian).
    % , then with probability 1 we have that $\spark(\rmW_i) = n_i + 1$, implying that no $n_i$ columns are linearly-dependent. Therefore, 
    If the representations of a signal $\rvx = G(\rvz^*)$ satisfy:
    \begin{enumerate}[label=(\roman*)]
        \item $s_L = \norm{\rvx_L}_0 < \frac{n + 1}{2}$.
        \item $s_i = \norm{\rvx_i}_0 < \frac{s_{i+1} + 1}{2}$, for all $i \in \{1, \ldots, L-1\}$.
        \item $s_{1} = \norm{\rvx_{1}}_0 \geq n_0$,
    \end{enumerate}
    then, with probability 1 the inverse problem has a unique solution that meets these sparsity conditions.
    \label{cor:random}
\end{corollary}

The above corollary states that the number of nonzero elements should expand by a factor of at least $2$ between layers to ensure a unique global minimum of the inverse problem.\footnote{In many practical architectures, the last non-linear activation, e.g. $\tanh$, does not promote sparsity, enabling a decrease in the dimensionality of the image compared to the last representation vector (i.e. we allow for $n_L>n$).} As presented in Section \ref{sec:experiments_random}, these conditions are very effective in predicting whether the generative process is invertible or not, regardless of the recovery algorithm used.

\section{Pursuit Guarantees}

In this section we provide an inversion algorithm supported by reconstruction guarantees for the noiseless and noisy settings. To reveal the potential of our approach, we first discuss the performance of an Oracle, in which the true supports of all the hidden layers are known, and only their values are missing. This estimation, which is described in Algorithm \ref{alg:oracle}, is performed by a sequence of simple linear projections on the known supports starting by estimating $\rvx_L$ and ending with the estimation of $\rvz$. Note that already in the first step of estimating $\rvx_L$, we can realize the advantage of utilizing the inherent sparsity of the hidden layers. Here, the reconstruction error of the Oracle is proportional to $s_L = \norm{\rvx_L}_0$, whereas solving a least square problem, as suggested in \cite{lei2019inverting}, results with an error that is proportional to $n_L$. More details on this estimator appears in Appendix \ref{app:oracle}.

\begin{algorithm}[H]
\caption{The Oracle} \label{alg:oracle}
\textbf{Input:} $\rvy = G(\rvz) + \rve \in \R^n$, and \emph{supports of each layer} $\{\S_i\}_{i=1}^L$.\\
\textbf{Algorithm:} Set $\rvz \gets \argmin_{\rvz} \frac{1}{2}\norm{\rvy - \phi\left( \left(\prod_{i=L}^{0} \rmW_i^{\hS_{i+1}}\right) \rvz \right)}_2^2$.
\end{algorithm}

In what follows, we propose to invert the model by solving sparse coding problems layer-by-layer, while leveraging the sparsity of all the intermediate feature vectors. Specifically, Algorithm \ref{alg:layered_bp} describes a layered Basis-Pursuit approach, and Theorem \ref{thm:layered_bp} provides reconstruction guarantees for this algorithm. The proof of this theorem is given in Appendix \ref{app:layered_bp}. In Corollary \ref{cor:realizable} we provide guarantees for this algorithm when inverting non-random generative models in the noiseless case.

\begin{algorithm}[H]
\caption{Layered Basis-Pursuit} \label{alg:layered_bp}
\textbf{Input:} $\rvy = G(\rvz) + \rve \in \R^n$, where $\norm{\rve}_2\leq \epsilon$, and sparsity levels $\{s_i\}_{i=1}^L$.\\
% \textbf{Initialization:} Set $\hrvx_{L+1} = \rvy$ , \,$\epsilon_{L+1} = \epsilon$, \,$\hS_{L+1} = \{1,\ldots,n\}$.\\
\textbf{First step:} $\hrvx_L = \argmin_\rvx ~ \frac{1}{2}\norm{\phi^{-1}(\rvy) - \rmW_L \rvx}_2^2 + \lambda_L \norm{\rvx}_1$, \,with $\lambda_L = 2 \ell \epsilon$.\\
Set $\hS_L = \Support(\hrvx_L)$ and $\epsilon_L = \frac{(3+\sqrt{1.5}) \sqrt{s_L}}{\min_j \norm{\rvw_{L,j}}_2} \ell \epsilon$.\\
\textbf{General step:} For any layer $i =L-1,\ldots,1$ execute the following:
\begin{enumerate}
    \item $\hrvx_i = \argmin_\rvx ~ \frac{1}{2}\norm{\hrvx_{i+1}^{\hS_{i+1}} - \rmW_i^{\hS_{i+1}} \rvx}_2^2 + \lambda_i \norm{\rvx}_1$, \,with $\lambda_i = 2 \epsilon_{i+1}$.
    \item Set $\hS_i = \Support(\hrvx_i)$ and $\epsilon_i = \frac{(3+\sqrt{1.5}) \sqrt{s_i}}{\min_j \norm{\rvw_{i,j}^{\hS_{i+1}}}_2} \epsilon_{i+1}$.
\end{enumerate}
\textbf{Final step:} Set $\hrvz = \argmin_\rvz ~ \frac{1}{2}\norm{\hrvx_1^{\hS_1} - \rmW_0^{\hS_1} \rvz}_2^2$.
\end{algorithm}

\begin{definition}[Mutual Coherence of Submatrix]
    Define $\mu_s(\rmW)$ as the maximal mutual coherent of any submatrix of $\rmW$ with $s$ rows:
    \begin{equation}
        \mu_s(\rmW) = \max_{|\S|=s} \mu(\rmW^{\S}).
    \end{equation}
\end{definition}

\begin{theorem}[Layered Basis-Pursuit Stability] \label{thm:layered_bp}
    Suppose that $\rvy = \rvx + \rve$, where $\rvx = G(\rvz)$ is an unknown signal with known sparsity levels $\{s_i\}_{i=1}^L$, and $\norm{\rve}_2 \leq \epsilon$. Let $\ell$ be the Lipschitz constant of $\phi^{-1}$ and define $\epsilon_{L+1} = \ell \epsilon$. If in each midlayer $i \in \{1,\ldots,L\}$ the sparsity level satisfies $s_i < \frac{1}{3\mu_{s_{i+1}}(\rmW_i)}$, then,
    \begin{itemize}
        \item The support of $\hrvx_i$ is a subset of the true support, $\hS_i \subseteq \S_i$;
        \item The vector $\hrvx_i$ is the unique solution for the basis-pursuit;
        \item The midlayer's error satisfies $\norm{\hrvx_i - \rvx_i}_{2} < \epsilon_i$.
    \end{itemize}
    In addition, denoting $\varphi = \lmin((\rmW_0^{\hS_1})^T \rmW_0^{\hS_1}) > 0$, the recovery error on the latent space is upper bounded by
    \begin{equation}
        \norm{\hrvz - \rvz}_{2} < \frac{\epsilon \ell}{\sqrt{\varphi}} \prod_{i=1}^L \frac{(3 + \sqrt{1.5}) \sqrt{s_j}}{\min_j \norm{\rvw^{\hS_{i+1}}_{i, j}}_2}.
    \end{equation}
\end{theorem}

\begin{corollary}[Layered Basis-Pursuit -- Noiseless Case]
\label{cor:realizable}
    Let $\rvx = G(\rvz)$ with sparsity levels $\{s_i\}_{i=1}^L$, and assume that $s_i < 1/3\mu_{s_{i+1}}(\rmW_i)$ for all $i \in \{1,\ldots,L\}$, and that $\varphi = \lmin((\rmW_0^{\hS_1})^T \rmW_0^{\hS_1}) > 0$. Then Algorithm \ref{alg:layered_bp} recovers the latent vector $\hrvz = \rvz$ perfectly.
\end{corollary}

\section{The Latent-Pursuit Algorithm}

While Algorithm \ref{alg:layered_bp} provably inverts the generative model, it only uses the non-zero elements $\rvx_{i+1}^{\hS_{i+1}}$ to estimate the previous layer $\rvx_{i}$. Here we present another algorithm, the Latent-Pursuit algorithm, which expands the Layered Basis-Pursuit algorithm by imposing two additional constraints. First, the Latent-Pursuit uses not only the non-zero elements of the subsequent layer, but the zero entries as well. These turn into inequality constraints, $\rmW_i^{\S_{i+1}^c} \rvx_i \leq \rvzero$, that emerge from the $\relu$ activation. Second, recall that the $\relu$ activation constrains the midlayers to have nonnegative values, $\rvx_i\ge 0$. Finally, we refrain from applying the inverse activation function $\phi^{-1}$ on the signal since in practical cases, such as $\tanh$, this inversion might be unstable. The proposed algorithm is composed of three parts: (i) the image layer; (ii) the middle layers; and (iii) the first layer. In what follows we describe each of these steps.

\begin{figure}[t]
    \centering
    \includegraphics[width=1\textwidth]{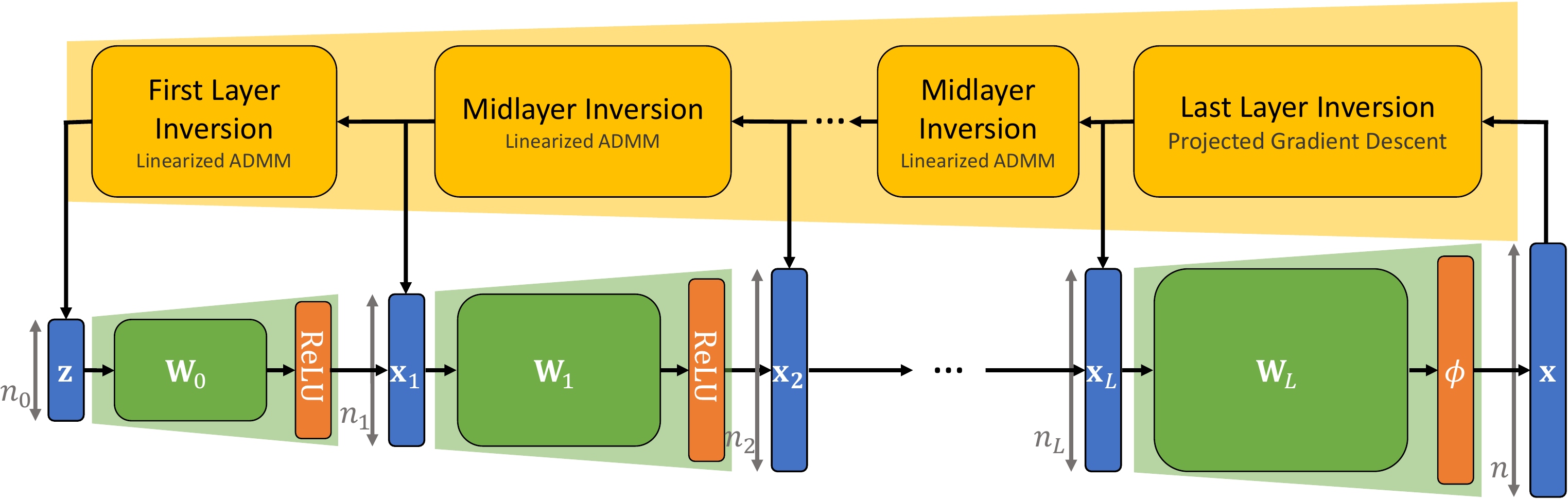}
    \caption{The Latent-Pursuit inverts the generative model layer-by-layer as described in Algorithm \ref{alg:latent_pursuit}. This recovery algorithm is composed of three steps: (1) last layer inversion by Algorithm \ref{alg:last_layer}; (2) midlayer inversions using Algorithm \ref{alg:midlaye_pursuit}; and (3) first layer inversion via Algorithm \ref{alg:midlaye_pursuit} with the $\rvx$-step detailed in \Eqref{eq:ADMM_update_x}.}
    \label{fig:latent_pursuit}
\end{figure}

We start our description of the algorithm with the inversion of the last layer, i.e. the image layer. Here we need to solve
\begin{equation}\label{eq:last_layer}
    \rvx_L = \argmin_{\rvx} \frac{1}{2}\norm{\rvy - \phi(\rmW_L \rvx)}_2^2 + \lambda_L \rvone^T \rvx, ~\st~ \rvx \geq \rvzero,
\end{equation}
where $1^T \rvx$ represents an $\ell_1$ regularization term under the nonnegative constraint. Assuming that $\phi$ is smooth and strictly monotonic increasing, this problem is a smooth convex function with separable constraints, and therefore, it can be solved using a projected gradient descent algorithm. In particular, in Algorithm \ref{alg:last_layer} we solve this problem using FISTA (Nesterov's acceleration) \cite{beck2009fast}.

\begin{algorithm}[H]
\caption{Latent Pursuit: Last Layer Inversion} \label{alg:last_layer}
\textbf{Input:} $\rvy\in\sR^n, K\in\sN, \lambda_L \geq 0, \mu\in (0, \frac{2}{\ell})$, where $\phi(\cdot)$ is $\ell$-smooth and strictly monotonic increasing.\\
\textbf{Initialization:} $\rvu^{(0)} \gets \rvzero, \rvx_L^{(0)} \gets \rvzero, t^{(0)}\gets 1$.\\
\textbf{General step:} for any $k=0,1,\ldots, K$ execute the following: 
\begin{enumerate}
    \item $\rvg \gets \rmW_L^T \phi'\left(\rmW_L \rvx_L^{(k)}\right)  \left[\phi\left(\rmW_L \rvx_L^{(k)}\right) - \rvy \right]$.
    \item $\rvu^{(k+1)} \gets \relu\left(\rvx_L^{(k)} - \mu\cdot(\rvg + \lambda_L \rvone)\right)$
    \item $t^{(k+1)} \gets \frac{1 + \sqrt{1 + 4 t^{(k)^2}}}{2}$
    \item $\rvx_L^{(k+1)} \gets \rvu^{(k+1)} + \frac{t^{(k)} - 1}{t^{(k+1)}} (\rvu^{(k+1)} - \rvu^{(k)})$
\end{enumerate}
\textbf{Return:} $\rvx_L^{(K)}$
\end{algorithm}

We move on to the middle layers, i.e. estimating $\rvx_i$ for $i\in\{1,\dots,L-1\}$. Here, both the approximated vector and the given signal are assumed to result from a $\relu$ activation function. This leads us to the following problem:
\begin{equation}
    \rvx_i = \argmin_{\rvx} \frac{1}{2}\norm{\rvx_{i+1}^{\hS} - \rmW_i^{\hS} \rvx}_2^2 + \lambda_i \rvone^T \rvx, ~\st~ \rvx \geq \rvzero,~ \rmW_i^{\hS^c} \rvx \leq \rvzero
\end{equation}
where $\hS = \hS_{i+1}$ is the support of the output of the layer to be inverted, and $\hS^c = \hS_{i+1}^c$ is its complementary. To solve this problem we introduce an auxiliary variable $\rva = \rmW_i^{\S^c} \rvx$, leading to the following augmented Lagrangian form:
\begin{equation} \label{eq:objective_midlayer}
\begin{split}
    \min_{\rvx, \rva, \rvu} ~ & \frac{1}{2}\norm{\rvx_{i+1}^{\S} - \rmW_i^{\S} \rvx}_2^2 + \lambda_i \rvone^T \rvx + \frac{\rho_i}{2}\norm{\rva - \rmW_i^{\S^c} \rvx +\rvu}_2^2 \\
    \st ~ & \rvx \geq \rvzero,~ \rva \leq \rvzero.
\end{split}
\end{equation}
This optimization problem could be solved using ADMM (alternating direction method of multipliers) \cite{boyd2011distributed}, however, it would require inverting a matrix of size $n_i \times n_i$, %both $\rmW_i^{\S^c}$ and $\rmW_i^{\S}$
which might be costly. Alternatively, we employ a more general method, called alternating direction \emph{proximal} method of multipliers \cite[Chapter 15]{beck2017first}, in which a quadratic proximity term, $\frac{1}{2}\norm{\rvx-\rvx^{(k)}}_{\rmQ}$, is added to the objective function (\Eqref{eq:objective_midlayer}). By setting 
\begin{equation}
    \rmQ = \alpha\rmI - {\rmW_i^{\S}}^T \rmW_i^{\S} + \beta \rmI - \rho_i {\rmW_i^{\S^c}}^T \rmW_i^{\S^c},
\end{equation}
with
\begin{equation}\label{eq:alpha_beta_cond}
    \alpha + \beta \geq \lmax({\rmW_i^{\S}}^T \rmW_i^{\S} + \rho_i {\rmW_i^{\S^c}}^T \rmW_i^{\S^c}),
\end{equation}
we get that $\rmQ$ is a positive semidefinite matrix. This leads to an algorithm that alternates through the following steps:
\begin{eqnarray}
    \rvx^{(k+1)} &\gets& \argmin_{\rvx} \frac{\alpha}{2} \norm{\rvx - \left( \rvx^{(k)} - \frac{1}{\alpha} {\rmW_i^{\S}}^T \left( \rmW_i^{\S} \rvx^{(k)} - \rvx_{i+1}^{\S} \right) \right) }_2^2 + \lambda_i \rvx + \nonumber\\
    & & \qquad \frac{\beta}{2} \norm{\rvx - \left( \rvx^{(k)} - \frac{\rho_i}{\beta} {\rmW_i^{\S^c}}^T \left( \rmW_i^{\S^c} \rvx^{(k)} - \rva^{(k)} - \rvu^{(k)} \right) \right) }_2^2 \label{eq:ADMM_update_x} \\
    & & \st \quad \rvx \geq \rvzero. \nonumber\\
    \rva^{(k+1)} &\gets& \argmin_{\rva} \frac{\rho_i}{2}\norm{\rva - \rmW_i^{\S^c} \rvx^{(k+1)} +\rvu^{(k)}}_2^2, ~ \st ~ \rva \leq \rvzero. \\
    \rvu^{(k+1)} &\gets& \rvu^{(k)} + \left( \rva^{(k+1)} - \rmW_i^{\S^c} \rvx^{(k+1)} \right).
\end{eqnarray}
Thus, the Linearized-ADMM algorithm, described in \ref{alg:midlaye_pursuit} is guaranteed to converge to the optimal solution of \Eqref{eq:objective_midlayer}.

\begin{algorithm}[H]
\caption{Latent Pursuit: Midlayer Inverse Problem Algorithm} \label{alg:midlaye_pursuit}
\textbf{Initialization:} $\rvx^{(0)} \in \R^{n_i}$, \,$\rvu^{(0)}, \rva^{(0)} \in \R^{s_{i+1}}$, \,$\rho_i > 0$, and $\alpha, \beta$ satisfying \Eqref{eq:alpha_beta_cond}.\\
\textbf{General step:} for any $k=0,1,\ldots$ execute the following:
\begin{enumerate}
    % \vspace{0.2cm}
    \item  $\rvx^{(k+1)} \gets \relu\big[ \rvx^{(k)} -  \frac{1}{\alpha + \beta} {\rmW_i^{\S}}^T ( \rmW_i^{\S} \rvx^{(k)} - \rvx_{i+1}^{\S} ) \\
    \qquad ~ \qquad ~ \qquad ~ \qquad ~ \qquad ~ \qquad
    - \frac{\rho_i}{\alpha + \beta} {\rmW_i^{\S^c}}^T ( \rmW_i^{\S^c} \rvx^{(k)} - \rva^{(k)} - \rvu^{(k)} ) - \frac{\lambda_i}{\alpha + \beta} \big].$
    \item $\rva^{(k+1)} \gets - \relu \left[ \rvu^{(k)} - \rmW_i^{\S^c} \rvx^{(k+1)} \right]$.
    \item $\rvu^{(k+1)} \gets \rvu^{(k)} + \rva^{(k+1)} - \rmW_i^{\S^c} \rvx^{(k+1)}$.
\end{enumerate}
\end{algorithm}

We now recovered all the hidden layers, and only the latent vector $\rvz$ is left to be estimated. For this inversion step we adopt a MAP estimator utilizing the fact that $\rvz$ is drawn from a normal distribution:
\begin{equation}
    \rvz = \argmin_{\rvz} \frac{1}{2}\norm{\rvx_1^{\S} - \rmW_{0}^{\S} \rvz}_2^2 + \frac{\gamma}{2}\norm{\rvz}_2^2, ~\st~ \rmW_{0}^{\S^c} \rvz \leq \rvzero,
\end{equation}
with $\gamma > 0$. This problem can be solved by the Linearized-ADMM algorithm described above, expect for the update of $\rvx$ (\Eqref{eq:ADMM_update_x}), which becomes:
\begin{equation} \label{eq:z_step}
\begin{split}
    \rvz^{(k+1)} \gets  & \argmin_{\rvz} \frac{\alpha}{2} \norm{\rvz - \left( \rvz^{(k)} - \frac{1}{\alpha} {\rmW_{0}^{\S}}^T \left( \rmW_{0}^{\S} \rvz^{(k)} - \rvx_{1}^{\S} \right) \right) }_2^2 + \\
    & \quad \frac{\beta}{2} \norm{\rvz - \left( \rvz^{(k)} - \frac{\rho_i}{\beta} {\rmW_{0}^{\S^c}}^T \left( \rmW_{0}^{\S^c} \rvz^{(k)} - \rva^{(k)} - \rvu^{(k)} \right) \right) }_2^2 + \frac{\gamma}{2}\norm{\rvz}_2^2.
\end{split}
\end{equation}
Equivalently, for the latent vector $\rvz$, the first step of Algorithm \ref{alg:midlaye_pursuit} is changed to to:
\begin{multline}
    \rvz^{(k+1)} \gets
    \frac{1}{\alpha + \beta + \gamma} \Big( (\alpha + \beta) \rvz^{(k)} - {\rmW_0^{\S}}^T ( \rmW_0^{\S} \rvz^{(k)} - \rvx_1^{\S} ) \\ - \rho_1 {\rmW_0^{\S^c}}^T ( \rmW_0^{\S^c} \rvz^{(k)} - \rva^{(k)} - \rvu^{(k)} ) \Big).
\end{multline}

Once the latent vector, $\rvz$, and all the hidden layers $\{\rvx_i\}_{i=1}^L$ are recovered, we propose an optional step to improve the final estimation. In this step, which we refer to as debiasing, we freeze the recovered supports and only optimize over the non-zero values in an end-to-end fashion. This is equivalent to computing the Oracle, only here the supports are not known, but rather estimated using the proposed pursuit. Algorithm \ref{alg:latent_pursuit} provides a short description of the entire proposed inversion method. 

\begin{algorithm}[H]
\caption{The Latent-Pursuit Algorithm} \label{alg:latent_pursuit}
\textbf{Initialization:} Set $\lambda_i > 0$ and $\rho_i > 0$.\\
\textbf{First step:} Estimate $\rvx_L$, i.e. solve \Eqref{eq:last_layer} using Algorithm \ref{eq:last_layer}.\\
\textbf{General step:} For any layer $i=L-1,\ldots, 1$, estimate $\rvx_i$ using Algorithm \ref{alg:midlaye_pursuit}.\\
\textbf{Final step:} Estimate $\rvz$ using Algorithm \ref{alg:midlaye_pursuit} but with the $x$-step described in \Eqref{eq:z_step}.\\
\textbf{Debiasing (optional):} Set $\rvz \gets \argmin_{\rvz} \frac{1}{2}\norm{\rvy - \phi\left( \left(\prod_{i=L}^{0} \rmW_i^{\hS_{i+1}}\right) \rvz \right)}_2^2$.
\end{algorithm}

\section{Numerical Experiments}
\label{sec:experiments}

We demonstrate the effectiveness of our approach through numerical experiments, where our goal is twofold. First we study random generative models and show the ability of the uniqueness claim above (Corollary \ref{cor:random}) to predict when both gradient descent and Latent-Pursuit algorithms fail to invert $G$ as there exists more than one solution to the inversion task. In addition, we show that in these random networks and under the conditions of Corollary \ref{cor:realizable}, the latent vector is perfectly recovered by both the Layered Basis-Pursuit and the Latent-Pursuit algorithm. Our second goal is to display the advantage of the Latent-Pursuit over the gradient descent alternative for trained generative models, in two settings: noiseless and image inpainting.

\subsection{Random Weights}
\label{sec:experiments_random}

\begin{figure}[t]
\centering
\begin{subfigure}{0.32\textwidth}
\centering
    \includegraphics[width=1\linewidth]{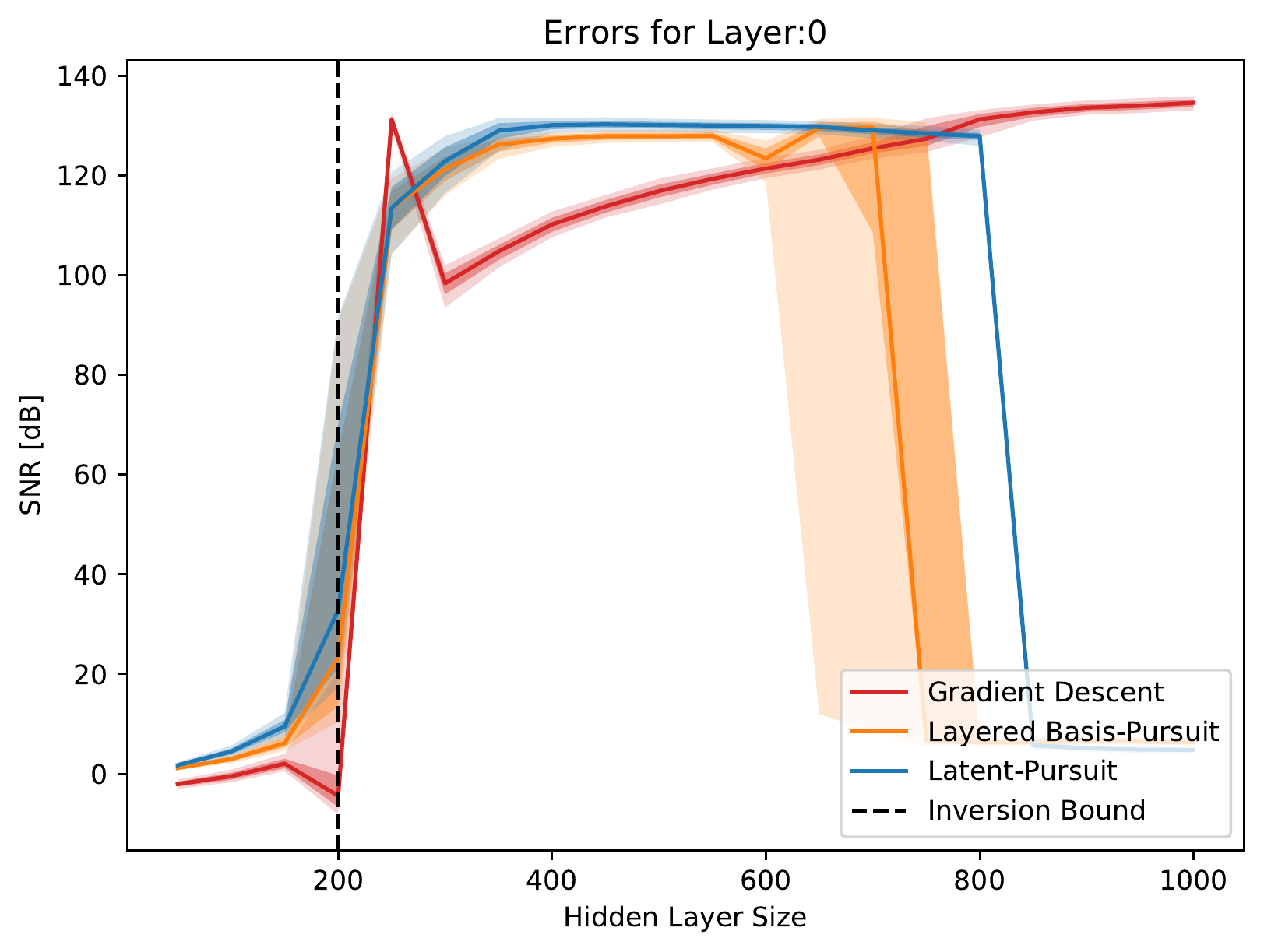}
    \caption{Latent vector $\rvz$}
\end{subfigure}
\begin{subfigure}{0.32\textwidth}
\centering
    \includegraphics[width=1\linewidth]{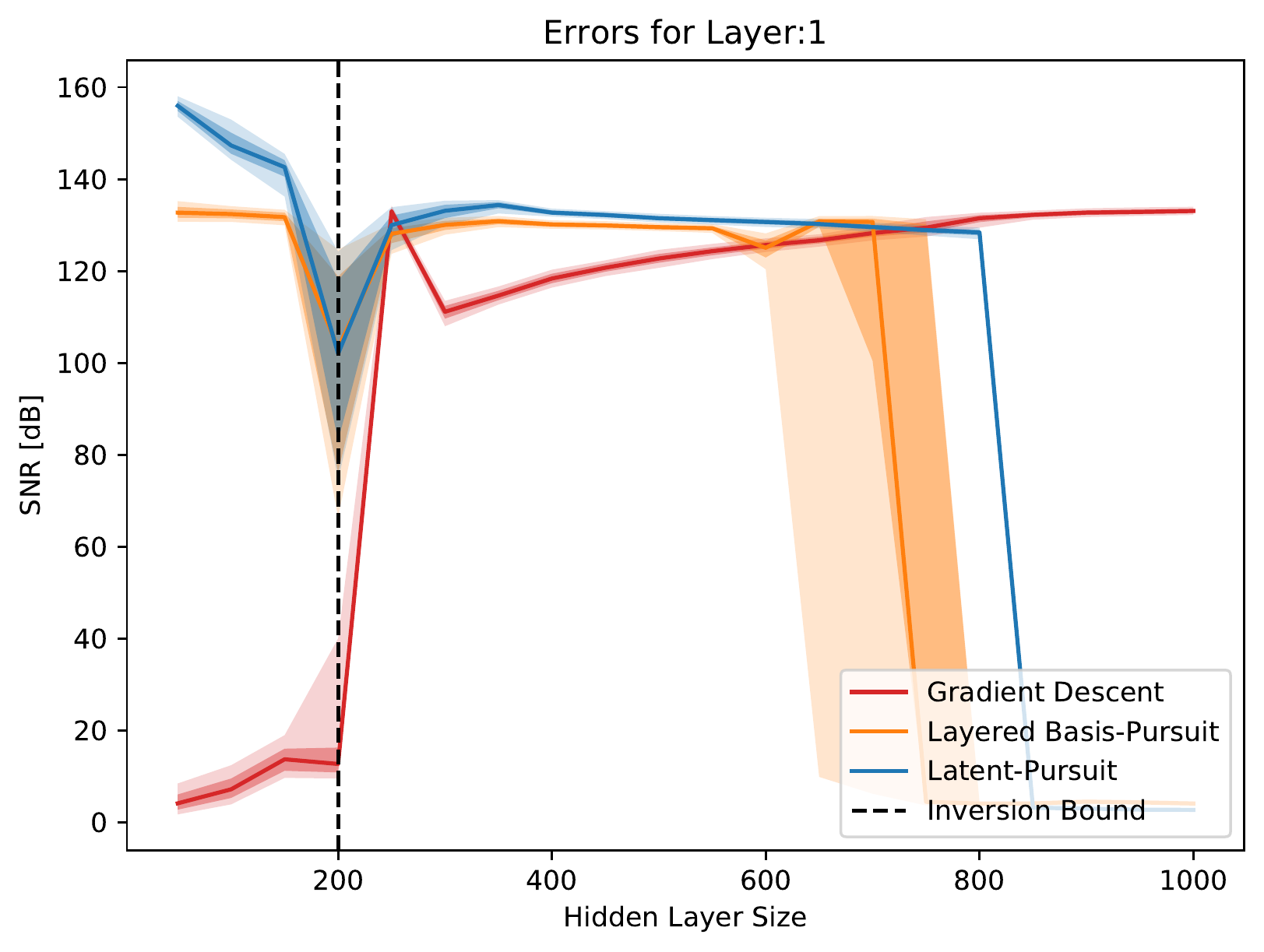}
    \caption{First hidden layer $\rvx_1$}
\end{subfigure}
\begin{subfigure}{0.32\textwidth}
\centering
    \includegraphics[width=1\linewidth]{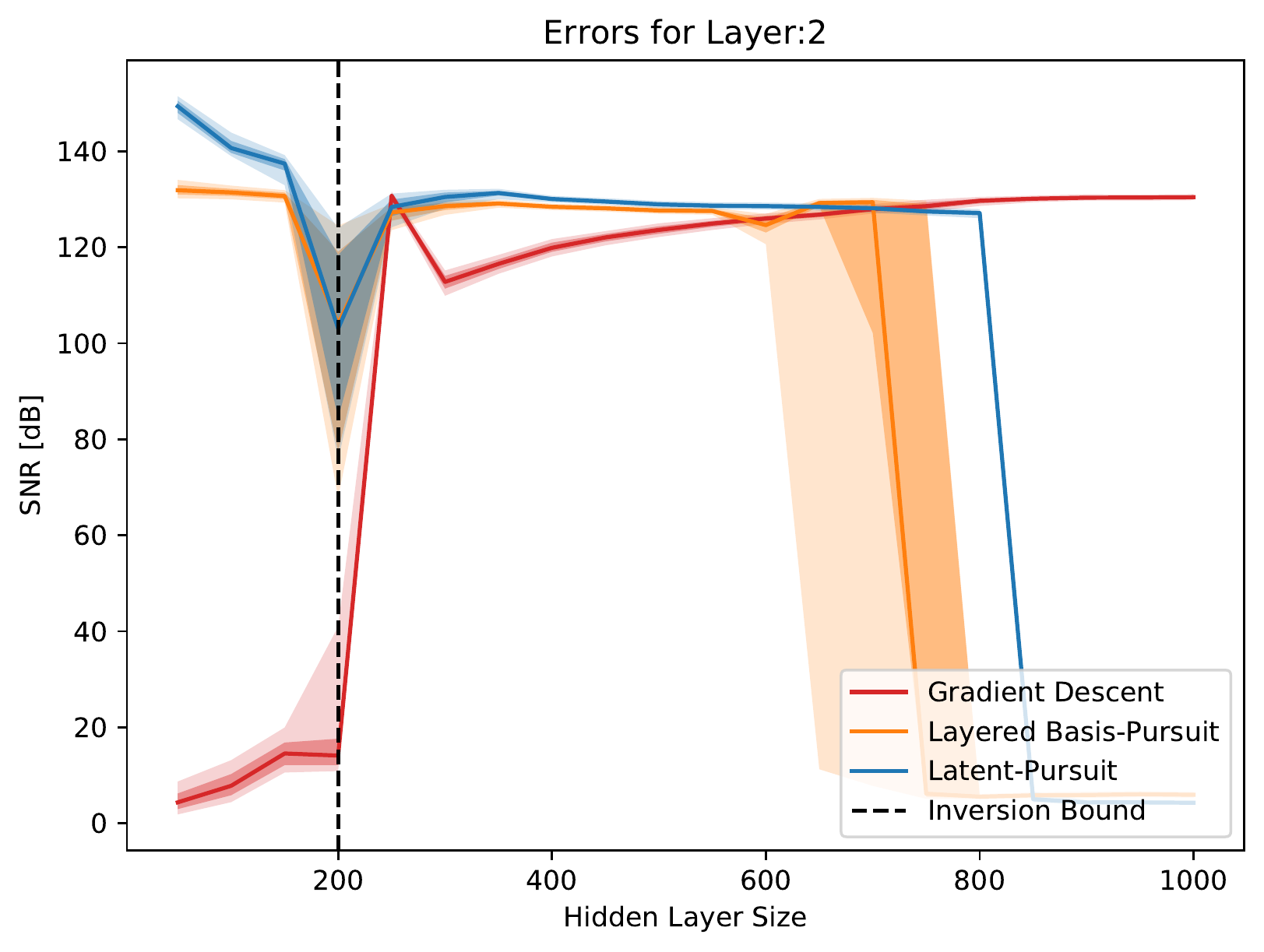}
    \caption{Image $G(\rvz)$}
\end{subfigure}
\\
\begin{subfigure}{0.32\textwidth}
\centering
    \includegraphics[width=1\linewidth]{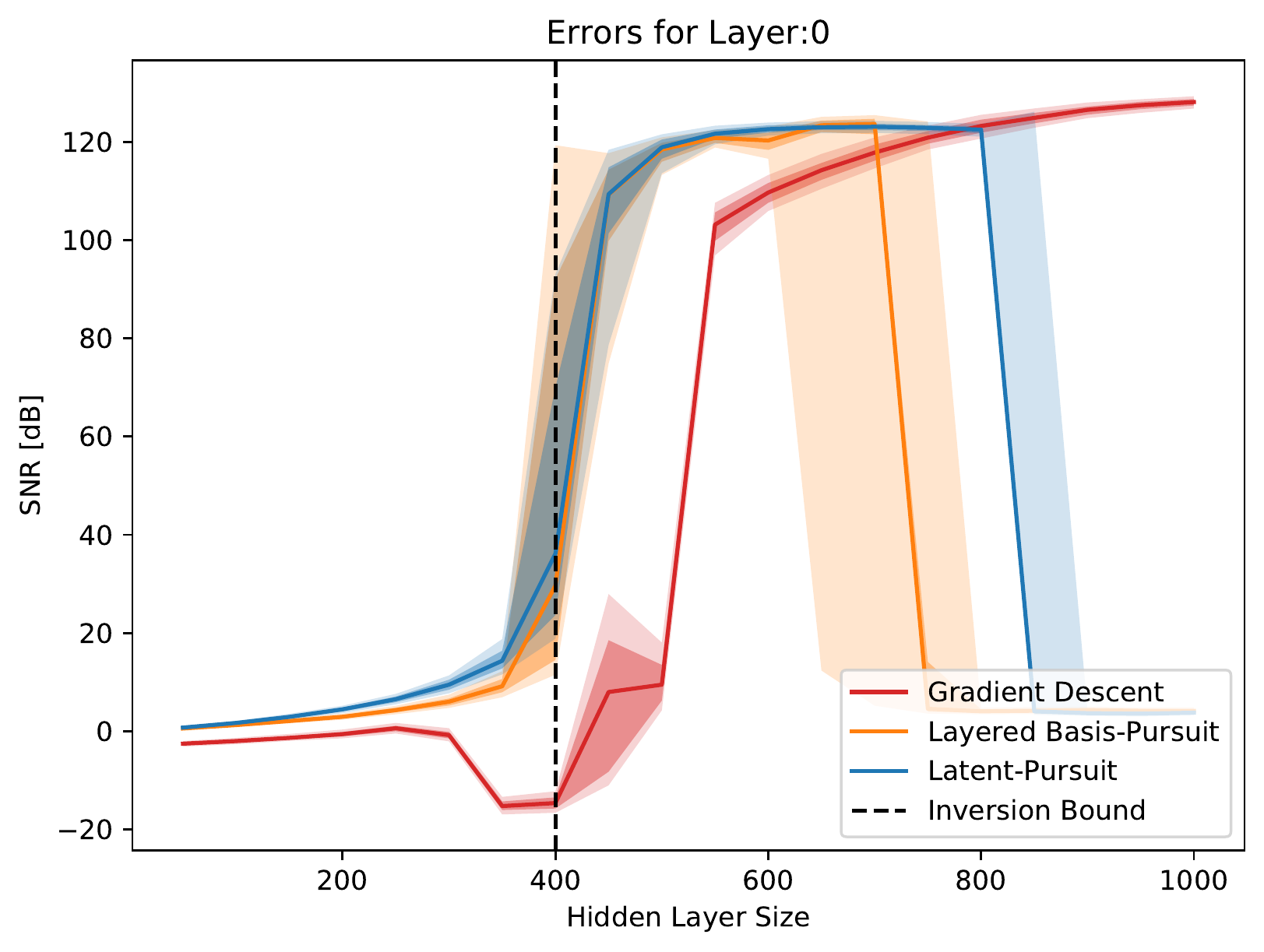}
    \caption{Latent vector $\rvz$}
\end{subfigure}
\begin{subfigure}{0.32\textwidth}
\centering
    \includegraphics[width=1\linewidth]{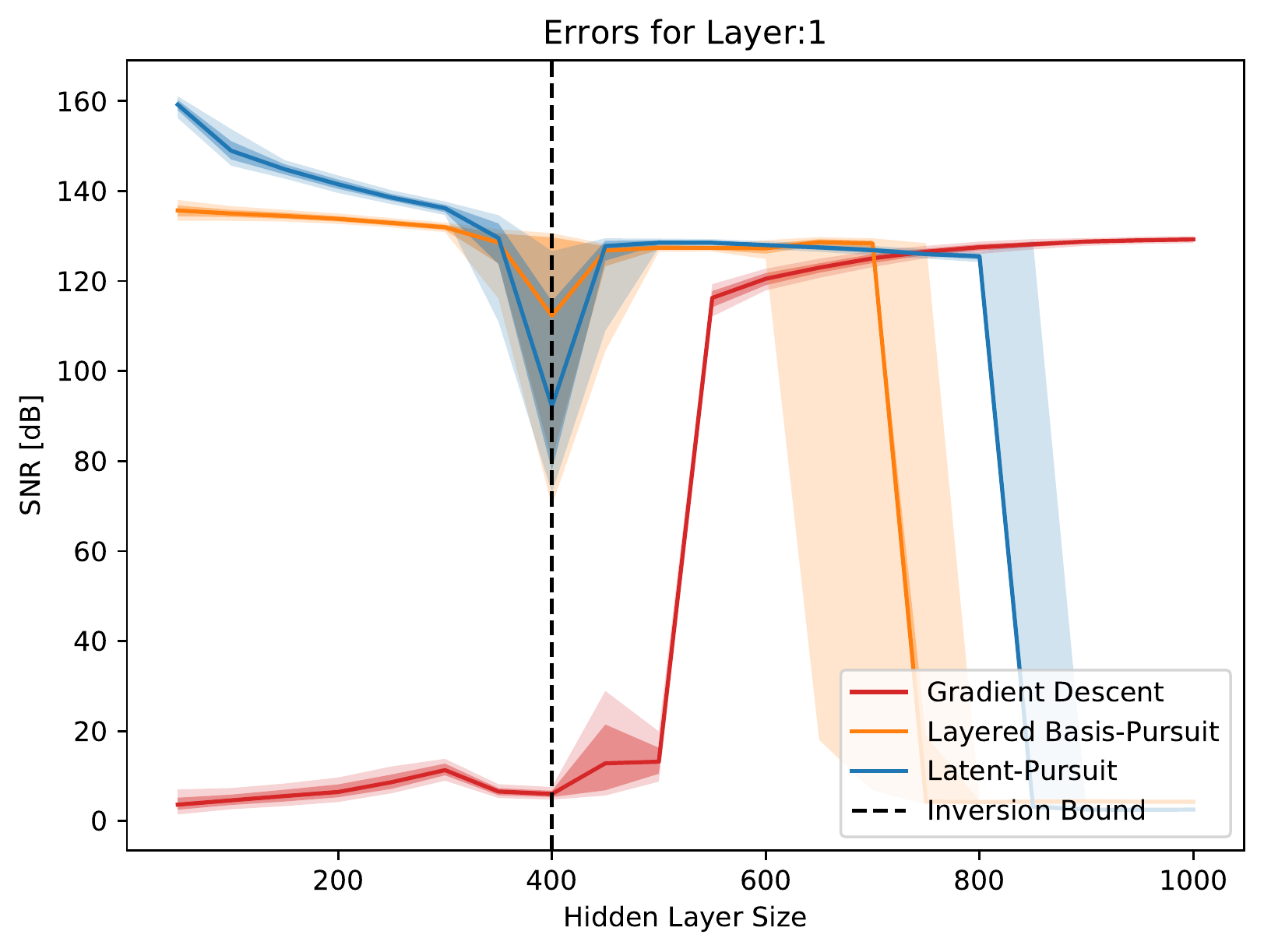}
    \caption{First hidden layer $\rvx_1$}
\end{subfigure}
\begin{subfigure}{0.32\textwidth}
\centering
    \includegraphics[width=1\linewidth]{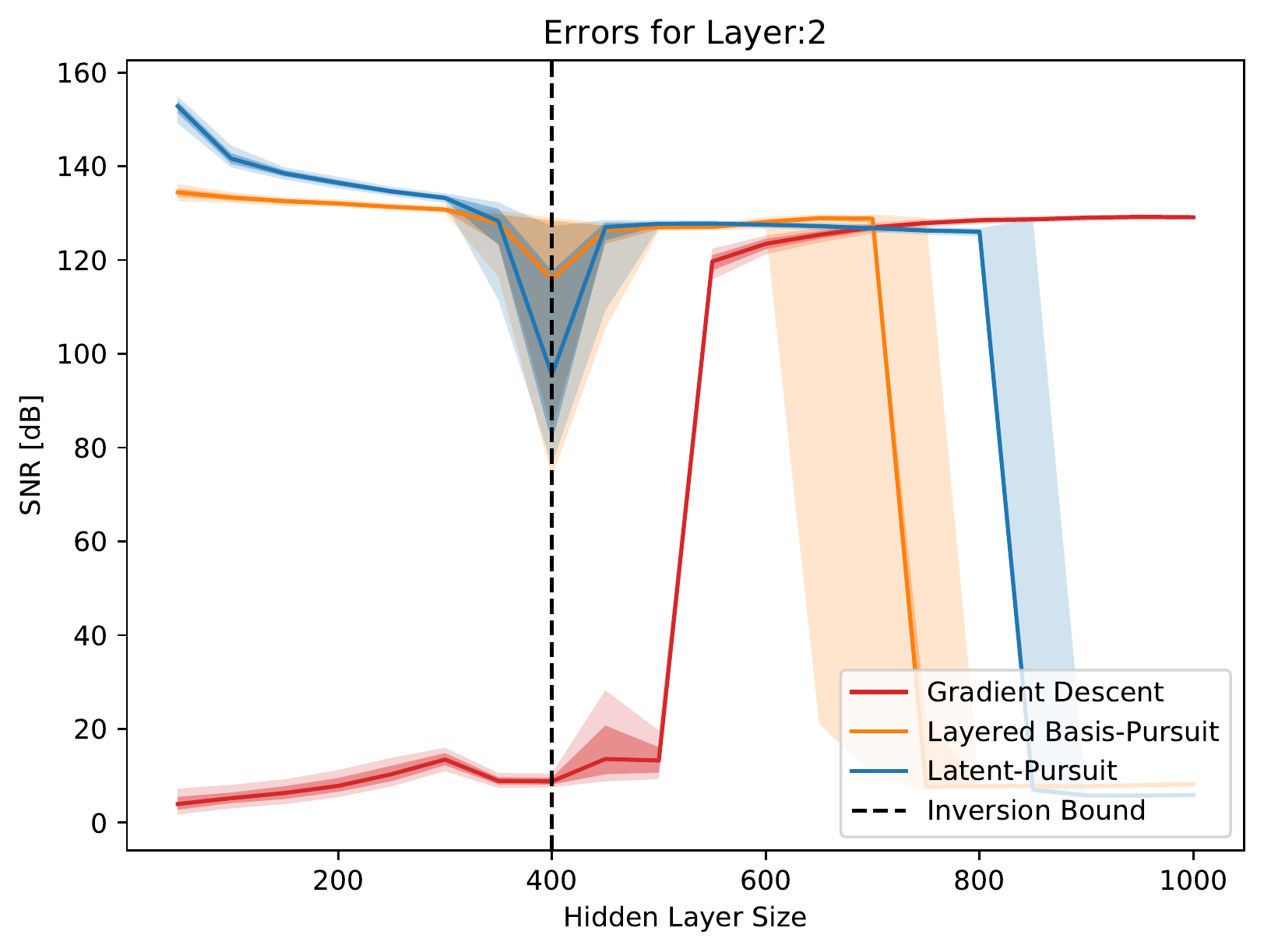}
    \caption{Image $G(\rvz)$}
\end{subfigure}
\caption{Gaussian iid weights: Recovery errors as a function of the hidden-layer size ($n_1$). The image space dimension is $625$, while the first row corresponds to $\rvz\in \R^{100}$ and the second to $\rvz\in \R^{200}$. These results support Corollary \ref{cor:random} stating that to guarantee a unique solution, the hidden layer cardinality $s_1 \approx \frac{n_1}{2}$ should be larger than the latent vector space and smaller than the cardinality of image space. Moreover, it supports Corollary \ref{cor:realizable} by showing that under the non-zero expansion condition, the Layered Basis-Pursuit (Algorithm \ref{alg:layered_bp}) and the Latent-Pursuit (Algorithm \ref{alg:latent_pursuit}) recover the original latent vector perfectly.}
\label{fig:random_matrices}
\end{figure}

In this part, we validate the above theorems on random generative models, by considering a framework similar to \cite{huang2018provably,lei2019inverting}. Here, the generator is composed of two layers:
\begin{equation}
    \rvx = G(\rvz) = \tanh (\rmW_2 \relu (\rmW_1 \rvz)),
\end{equation}
where the dimensions of the network are  $n = 625$, $n_1$ varies between $50$ to $1000$ and $n_0 \in \{ 100,200\}$. The weight matrices $\rmW_1$ and $\rmW_2$ are drawn from an iid Gaussian distribution. For each network, we test the performance of the inversion of $512$ random (realizable) signals in terms of SNR for all the layers, using gradient descent, Layered Basis-Pursuit (Algorithm \ref{alg:layered_bp}), and Latent-Pursuit (Algorithm \ref{alg:latent_pursuit}). For gradient descent, we use the smallest step-size from $\{1e-1, 1e0, 1e1, 1e2, 1e3, 1e4\}$ for $10,000$ steps that resulted with a gradient norm smaller than $1e-9$. For Layered Basis-Pursuit we use the best $\lambda_1$ from $\{1e-5, 7e-6, 3e-6, 1e-6, 0\}$, and for Latent-Pursuit, we use $\lambda_1=0$,  $\rho=1e-2$ and $\gamma=0$. In Layered Basis-Pursuit and Latent-Pursuit we preform a debiasing step in a similar manner to gradient descent. Figure \ref{fig:random_matrices} marks median results in the central line, while the ribbons show 90\%, 75\%, 25\%, and 10\% quantiles.

In these experiments the sparsity level of the hidden layer is approximately $50\%$, $s_1 = \norm{\rvx_1}_0 \approx \frac{n_1}{2}$, due to the weights being random.
In what follows, we split the analysis of the results of this experiment to three segments. Roughly, these segments are $s_1<n_0$, $n_0<s_1<n$, and $n<s_1$ as suggested by the theoretical results given in Corollary \ref{cor:random} and \ref{cor:realizable}.

In the first segment, Figure \ref{fig:random_matrices} shows that all three methods fail. Indeed, as suggested by the uniqueness conditions introduced in Corollary \ref{cor:random}, when $s_1 < n_0$, the inversion problem of the first layer does not have a unique global minimizer. The dashed vertical line in Figure \ref{fig:random_matrices} marks the spot where $\frac{n_1}{2} = n_0$. Interestingly, we note that the conclusions in \cite{huang2018provably,lei2019inverting}, suggesting that large latent spaces cause gradient descent to fail, are imprecise and valid only for fixed hidden layer size. This can be seen by comparing $n_0=100$ to $n_0=200$. As a direct outcome of our uniqueness study and as demonstrated in Figure \ref{fig:random_matrices}, gradient descent (and any other algorithm) fails when the ratio between the cardinalities of the layers is smaller than $2$. Nevertheless, Figure \ref{fig:random_matrices} exposes an advantage for using our approach over gradient descent. Note that our methods successfully invert the model for all the layers that follow the layer for which the sparsity assumptions do not hold, and fail only past that layer, since only then uniqueness is no longer guaranteed. However, since gradient descent starts at a random location, all the layers are poorly reconstructed.

For the second segment, we recall Theorem \ref{thm:layered_bp} and in particular Corollary \ref{cor:realizable}. There we have shown that Layered Basis-Pursuit and Latent-Pursuit are guaranteed to perfectly recover the latent vector as long as the cardinality of the midlayer $s_1=\norm{\rvx_1}_0$ satisfies $n_0 \leq s_1 \leq 1/ 3\mu(\rmW_1)$. Indeed, Figure \ref{fig:random_matrices} demonstrates the success of these two methods even when $s_1 \approx \frac{n_1}{2}$ is greater than the worst-case bound $1/ 3\mu(\rmW_1)$. Moreover, this figure demonstrates that Latent-Pursuit, which leverages additional properties of the signal, outperforms Layered Basis-Pursuit, especially when $s_1$ is large. 
Importantly, while the analysis in \cite{lei2019inverting} suggests that $n$ has to be larger than $n_1$, in practice, all three methods succeed to invert the signal even when $n_1 > n$.
This result highlights the strength of the proposed analysis that leans on the cardinality of the layers rather than their size. 

We move on to the third and final segment, where the size of hidden layer is significantly larger than the dimension of the image. Unfortunately, in this scenario the layer-wise methods fail, while gradient descent succeeds. Note that, in this setting, inverting the last layer solely is an ambitious (actually, impossible) task; however, since gradient descent solves an optimization problem of a much lower dimension, it succeeds in this case as well.
This experiment and the accompanied analysis suggest that a hybrid approach, utilizing both gradient descent and the layered approach, might be of interest. We defer a study of such an approach for future work.

\subsection{Trained Network}
To demonstrate the practical contribution of our work, we experiment with a generative network trained on the MNIST dataset. Our architecture is composed of fully connected layers of sizes 20, 128, 392, and finally an image of size $28\times 28=784$. The first two layers include batch-normalization\footnote{Note that after training, batch-normalization is a simple linear operation.} and a $\relu$ activation function, whereas the last one includes a piecewise linear unit \cite{nicolae2018plu}. We train this network in an adversarial fashion using a fully connected discriminator and spectral normalization \cite{miyato2018spectral}. 

We should note that images produced by fully connected models are typically not as visually appealing as ones generated by convolutional architectures. However, since the theory provided here focuses on fully connected models, this setting was chosen for the experimental section, similar to other previous work \cite{huang2018provably,lei2019inverting} that study the inversion process. 

%We should note that as in other theoretical works studying the inversion of generative models \cite{huang2018provably,lei2019inverting}, our generator is also based on fully connected architecture which is inferior to a convolutional one. While this may explain the imperfect synthesized digits, we chose to work with this structure, so as to have a better fit to the developed theory and algorithms in this paper.
% Our future work includes an expansion of these to convolutional setting, leveraging the theory and algorithms in \cite{papyan2017working,zisselman2019local,papyan2017convolutional}.

\paragraph{Network inversion: } We start with the noiseless setting and test our inversion algorithm and compare it to the oracle (which knows the exact support of each layer) and to gradient descent. To invert a signal and compute its reconstruction quality, we first invert the entire model and estimate the latent vector, and then, we feed this vector back to the model to estimate the hidden representations. 
% For a deeper analysis of the inversion process, please refer to Appendix \ref{app:layerwise_inversion}. 
For our algorithm we use $\rho=1e-2$ for all layers and $10,000$ iterations of debiasing. For gradient-descent run, we use $10,000$ iterations, momentum of $0.9$ and a step size of $1e-1$ that gradually decays to assure convergence. Overall, we repeat this experiment $512$ times.

Figure \ref{fig:trained_realizable_errors} demonstrates the reconstruction error for all the layers. First, we observe that the performance of our inversion algorithm is on par with those of the oracle. Moreover, not only does our approach performs much better than gradient descent, but in many experiments the latter fails utterly when trying to reconstruct the image itself. In Figures \ref{fig:trained_realizable_images_fail} and \ref{fig:trained_realizable_images_succ} we demonstrate successful and failure cases of the gradient-descent algorithm compared to our approach.

A remark regarding the run-time of these algorithms is in place. Using an Nvidia 1080Ti GPU, the proposed Latent-Pursuit algorithm took approximately $15$ seconds per layer to converge for a total of about $75$ seconds to complete, including the debiasing step for all $512$ experiments. On the other hand, gradient-descent took approximately $30$ seconds to conclude.

\begin{figure}[H]
\centering
\begin{subfigure}{0.49\textwidth}
\centering
    \includegraphics[width=1\linewidth]{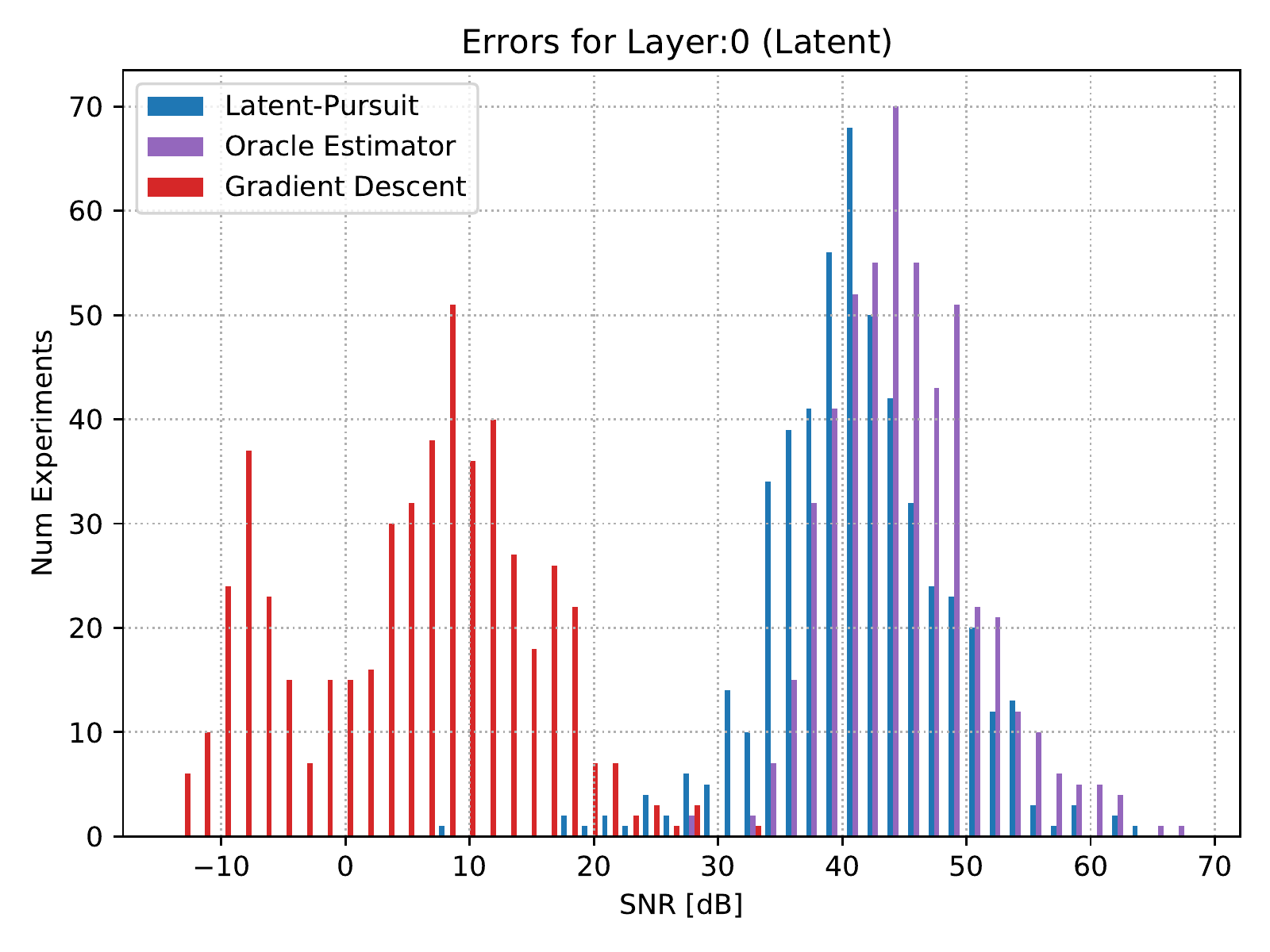}
    \caption{Latent vector $\rvz$}
\end{subfigure}
\begin{subfigure}{0.49\textwidth}
\centering
    \includegraphics[width=1\linewidth]{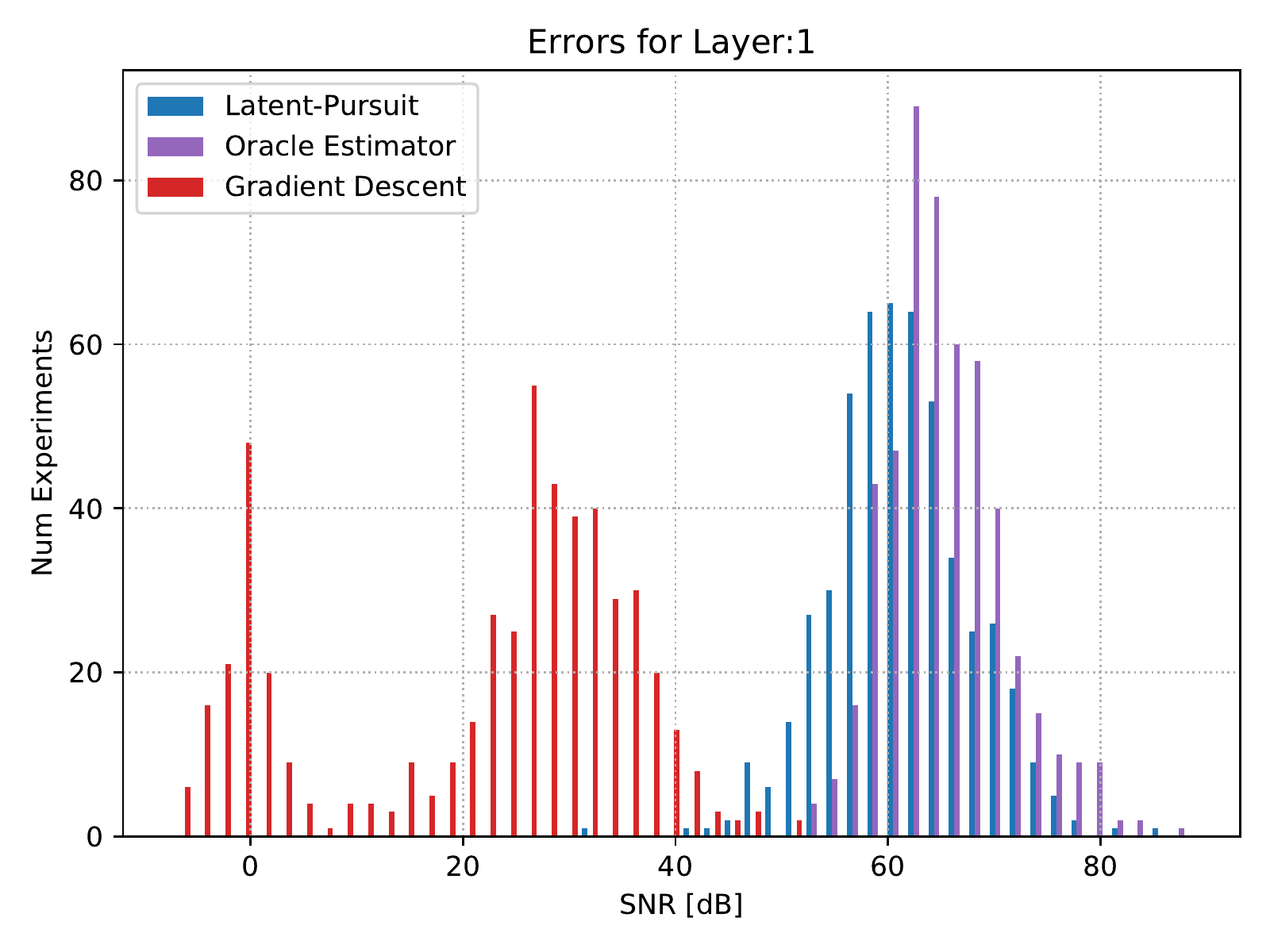}
    \caption{First hidden layer $\rvx_1$}
\end{subfigure}
\\
\begin{subfigure}{0.49\textwidth}
\centering
    \includegraphics[width=1\linewidth]{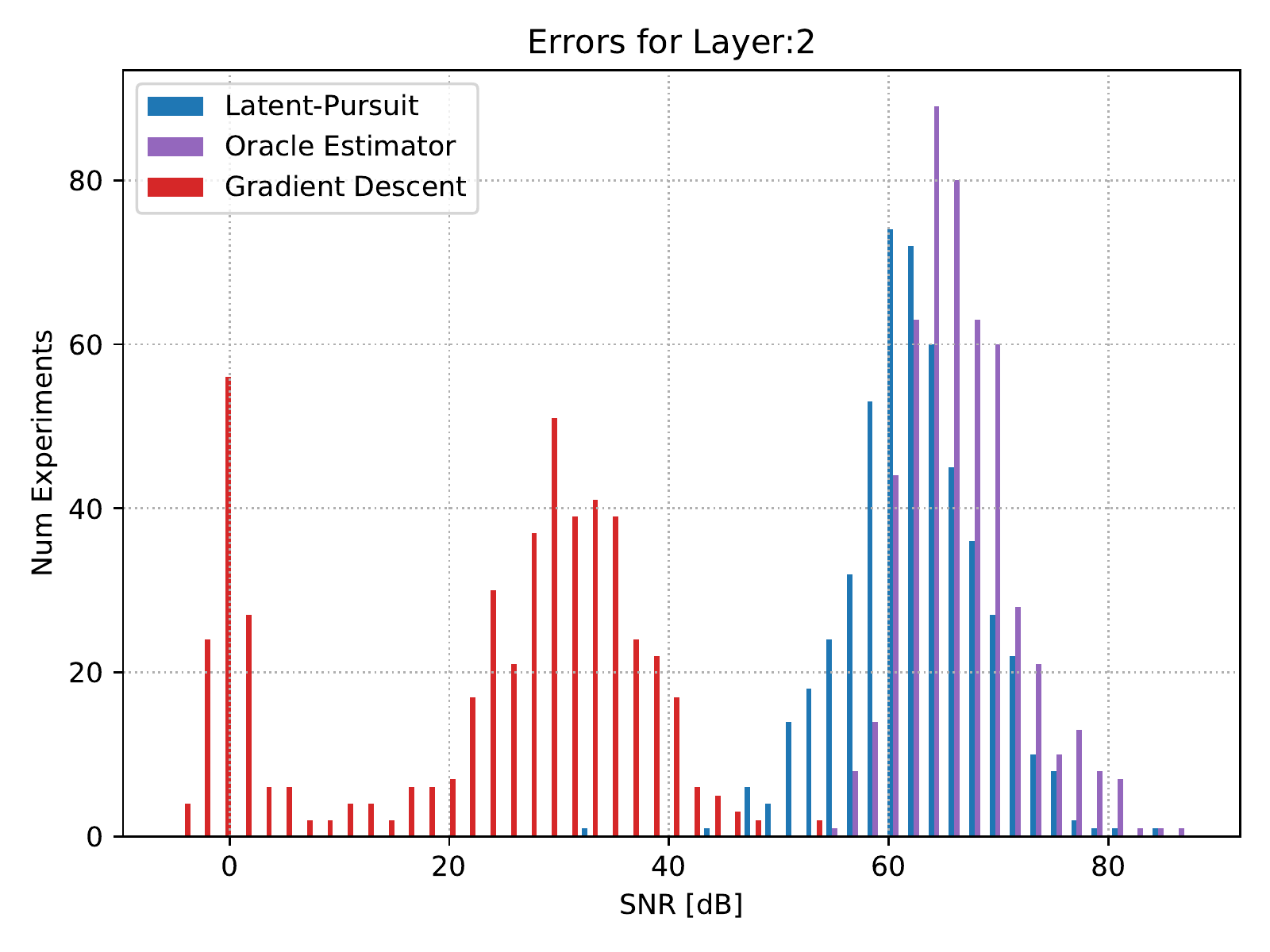}
    \caption{Second hidden layer  $\rvx_2$}
\end{subfigure}
\begin{subfigure}{0.49\textwidth}
\centering
    \includegraphics[width=1\linewidth]{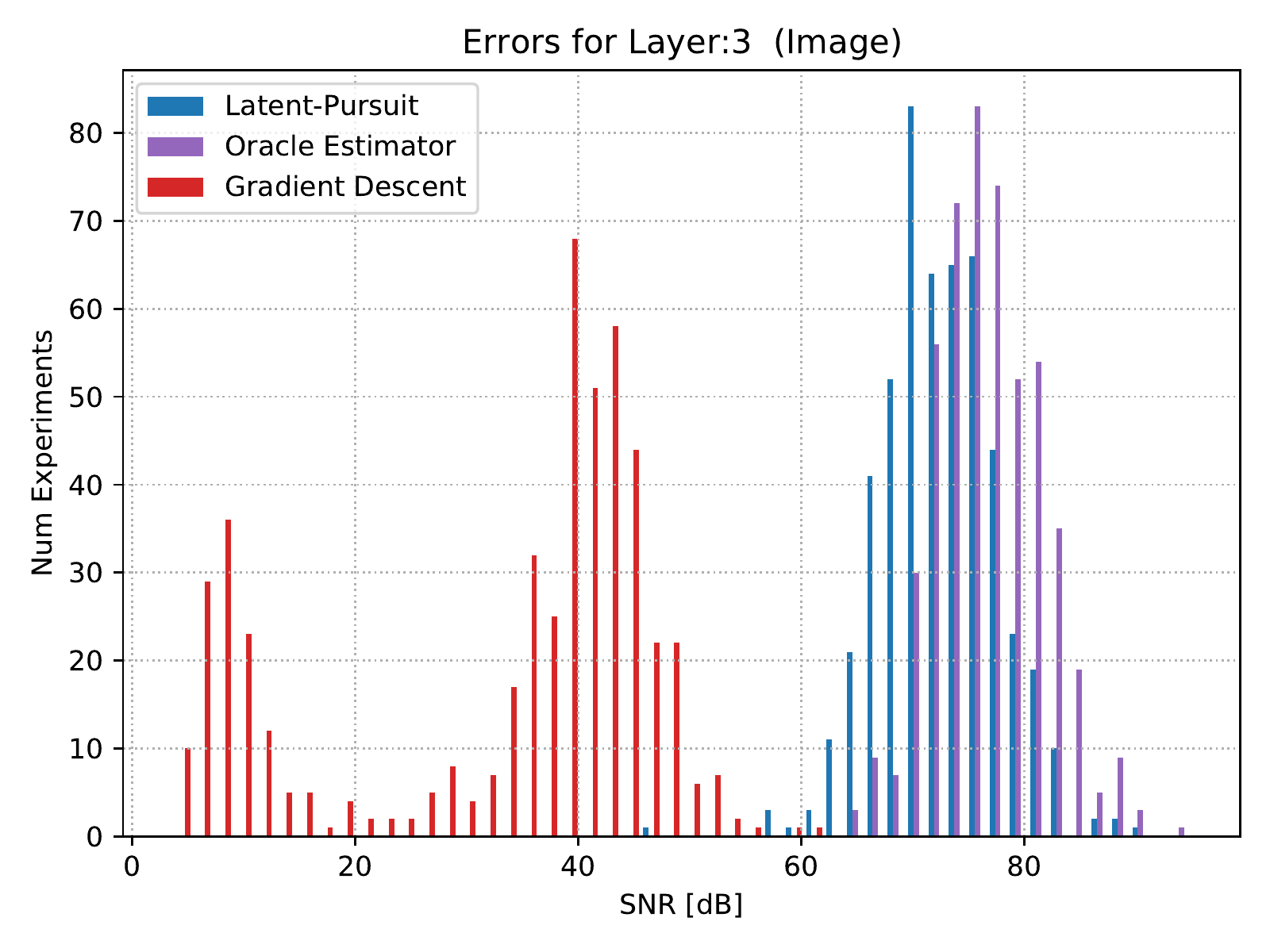}
    \caption{Image  $G(\rvz)$}
\end{subfigure}
\caption{Trained 3-Layers Model: Reconstruction error for all the layers for $512$ clean images. As can be observed, the Latent-Pursuit almost mimic the oracle (which knows all the supports), and outperforms gradient descent.}
\label{fig:trained_realizable_errors}
\end{figure}

\begin{figure}[H]
    \centering
    \begin{subfigure}{0.25\textwidth}
        Ground truth \vspace*{10pt} \\ \vspace*{10pt}Gradient descent \\ Our approach 
    \end{subfigure}
    \begin{subfigure}{0.7\textwidth}
        \includegraphics[trim={10 100 10 100},clip,width=1\linewidth]{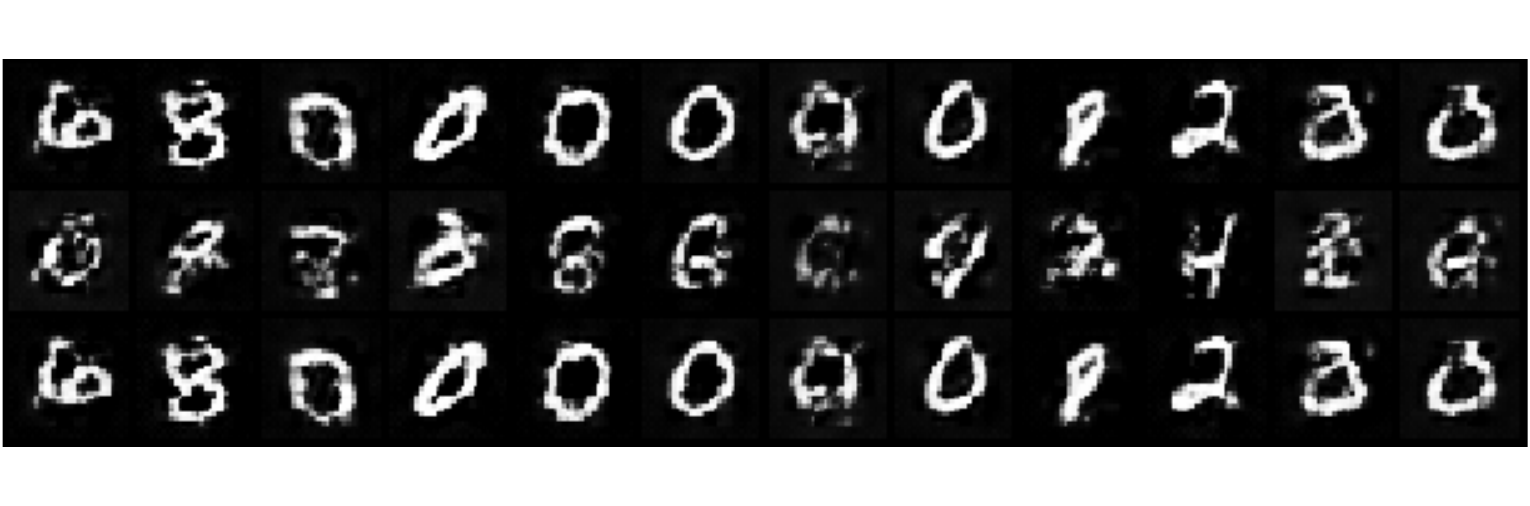}
    \end{subfigure}
    \caption{Reconstruction failures of gradient descent on clean images.}
    \label{fig:trained_realizable_images_fail}
\end{figure}

\begin{figure}[H]
    \centering
    \begin{subfigure}{0.25\textwidth}
        Ground truth \vspace*{10pt} \\ \vspace*{10pt}Gradient descent \\ Our approach 
    \end{subfigure}
    \begin{subfigure}{0.7\textwidth}
        \includegraphics[trim={10 100 10 100},clip,width=1\linewidth]{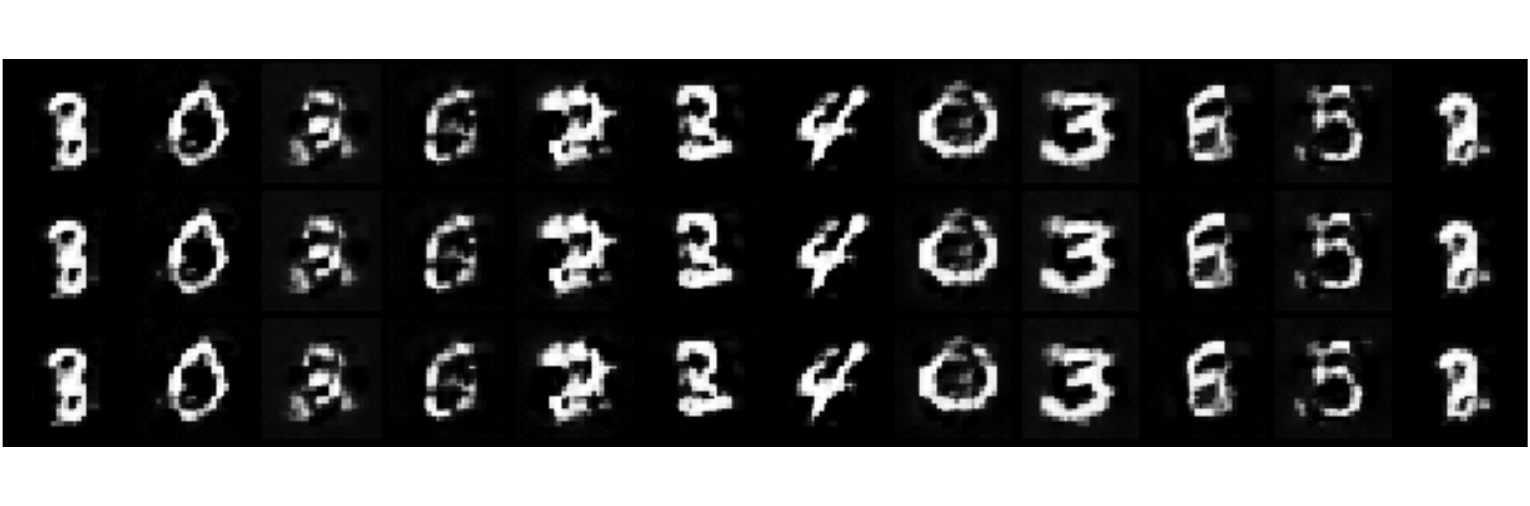}
    \end{subfigure}
    \caption{Successful reconstructions of gradient descent on clean images.}
    \label{fig:trained_realizable_images_succ}
\end{figure}

\paragraph{Image inpainting:} We continue our experiments with image inpainting, i.e. inverting the network and reconstructing a clean signal when only some of its pixels are known. First, we apply a random mask in which $45\%$ of the pixels are randomly concealed. Since the number of known pixels is still larger than the number of non-zero elements in the layer preceding it, our inversion algorithm usually reconstructs the image successfully as well as all the other hidden representations in the network. In this experiment, we perform slightly worse than the Oracle, which is not surprising considering the information disparity between the two. As for gradient descent, we see similar results to the ones in the previous experiment where no mask was applied.
Figures \ref{fig:inpainting_random_mse}-\ref{fig:inpainting_random_images_succ} demonstrate the performance of our approach compared to gradient descent in terms of SNR and image quality respectively.

\begin{figure}[H]
\centering
\begin{subfigure}{0.49\textwidth}
\centering
    \includegraphics[width=1\linewidth]{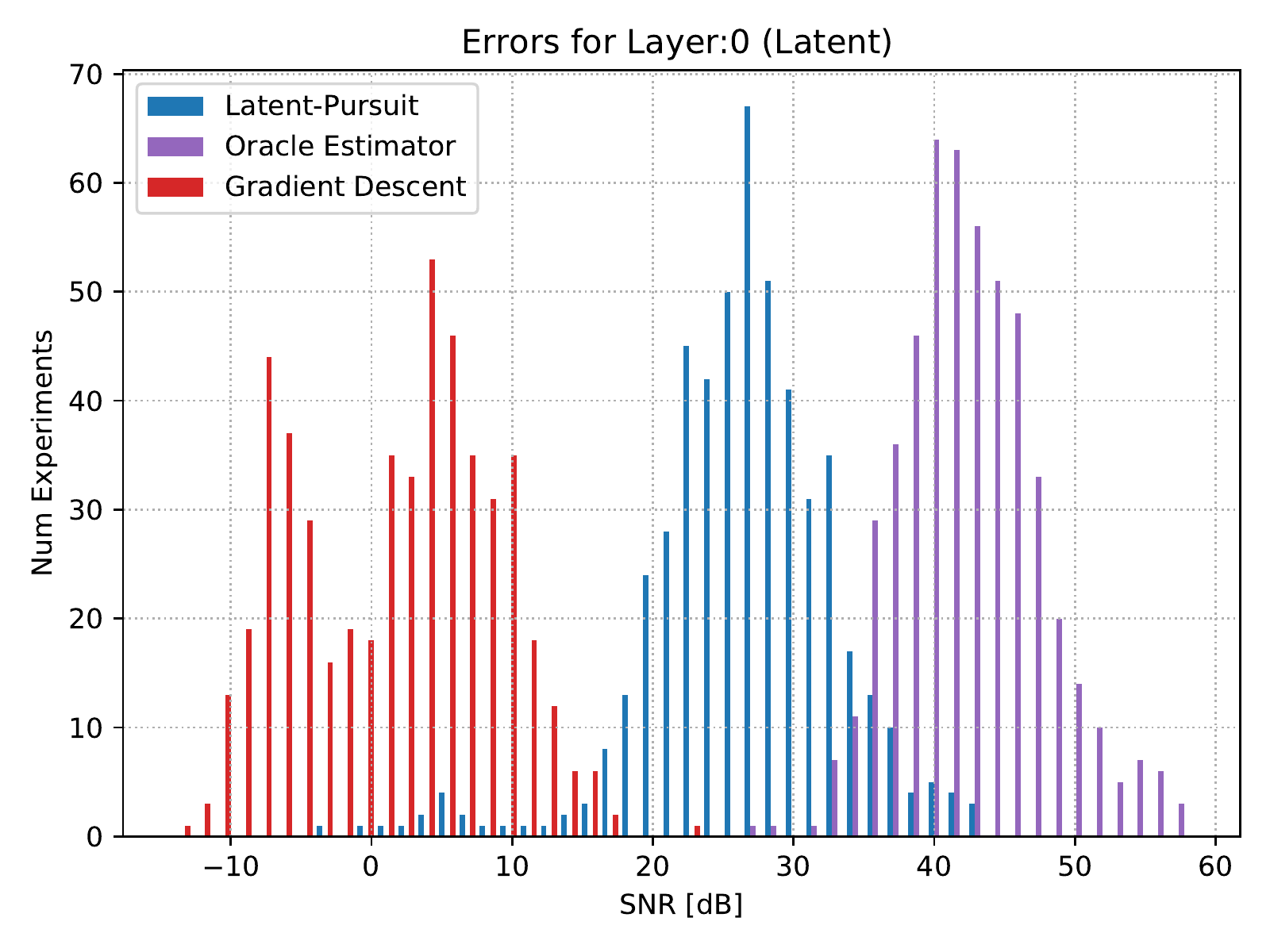}
    \caption{Latent vector $\rvz$}
\end{subfigure}
\begin{subfigure}{0.49\textwidth}
\centering
    \includegraphics[width=1\linewidth]{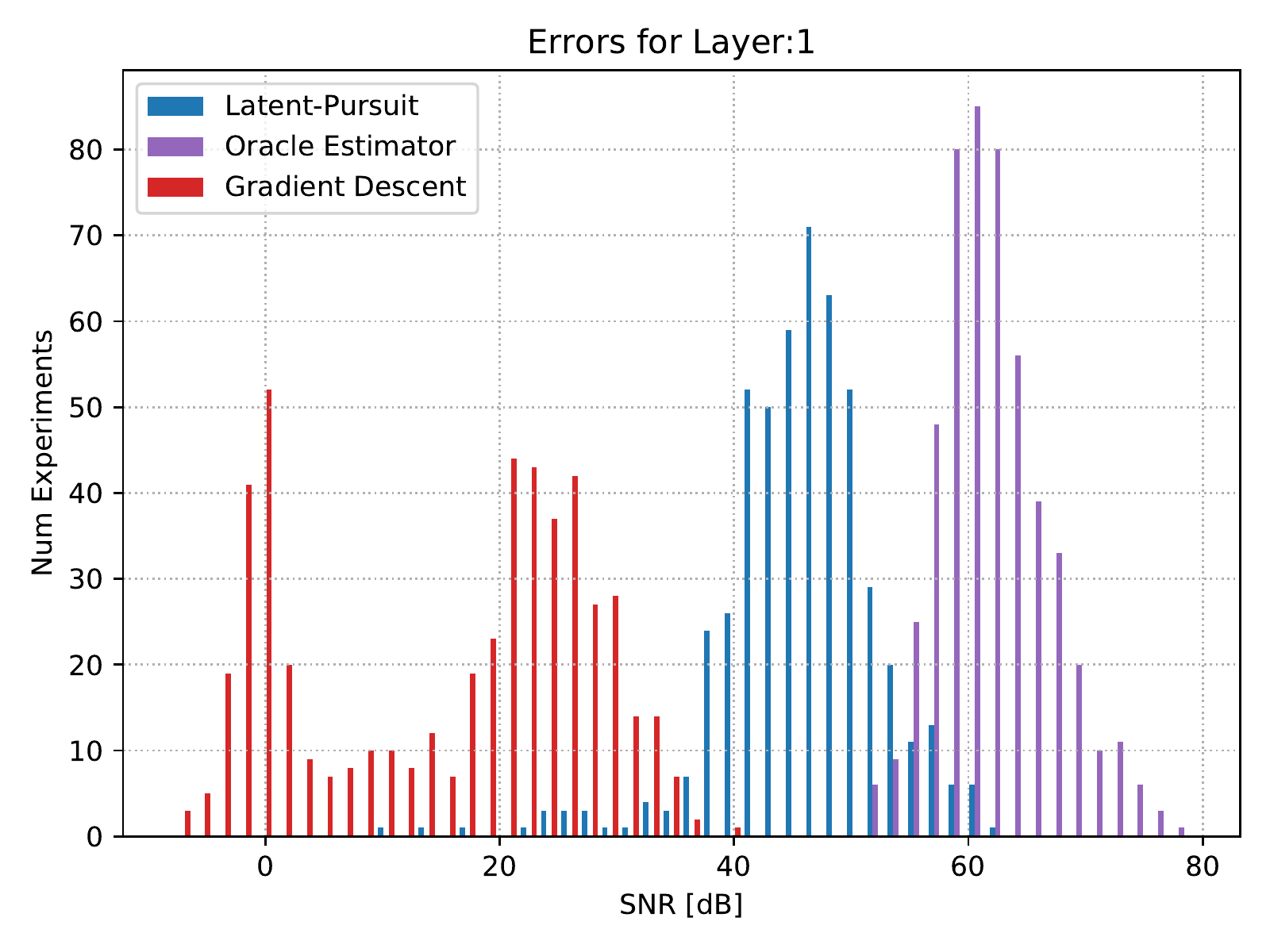}
    \caption{First hidden layer $\rvx_1$}
\end{subfigure}
\\
\begin{subfigure}{0.49\textwidth}
\centering
    \includegraphics[width=1\linewidth]{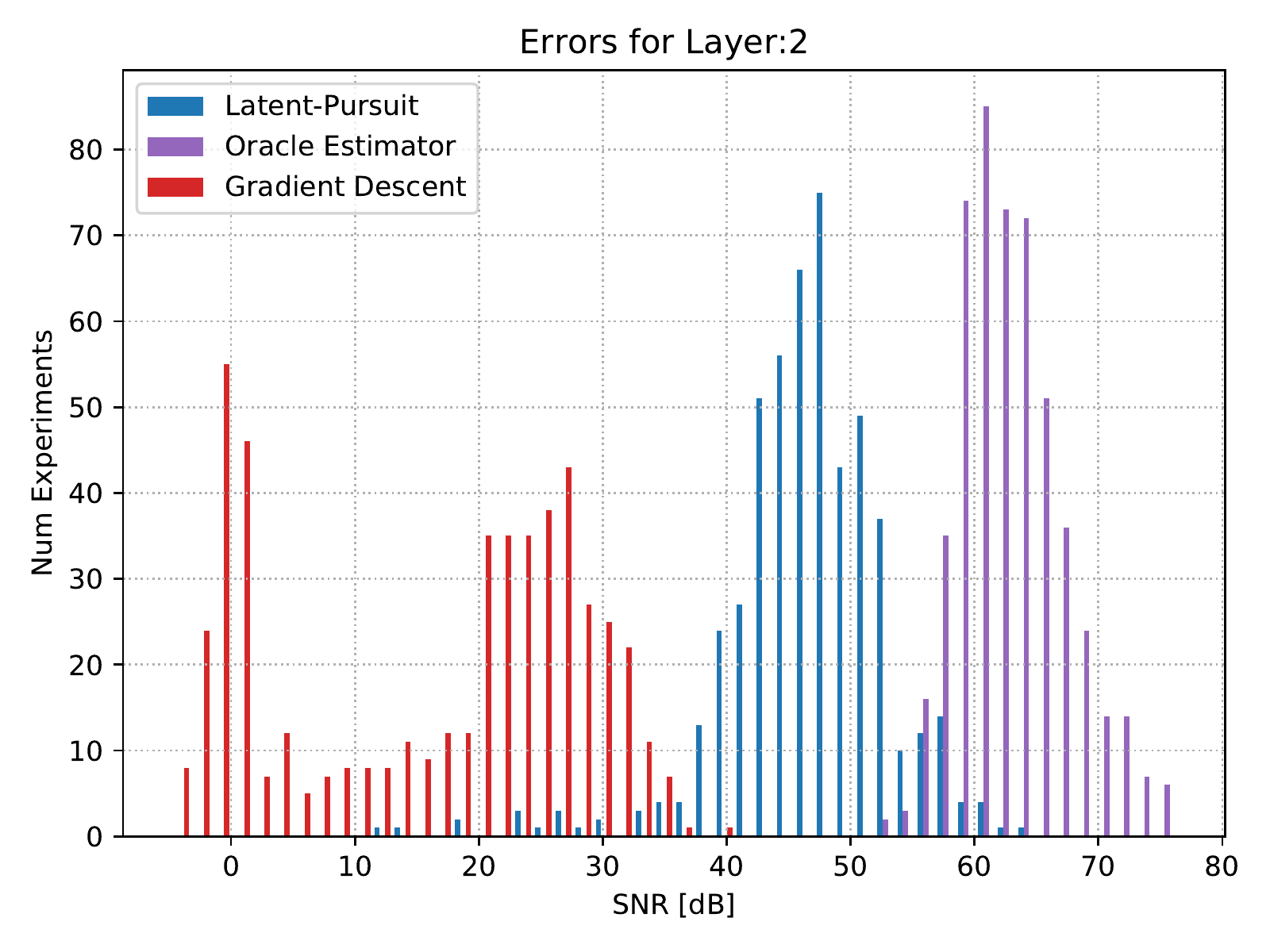}
    \caption{Second hidden layer  $\rvx_2$}
\end{subfigure}
\begin{subfigure}{0.49\textwidth}
\centering
    \includegraphics[width=1\linewidth]{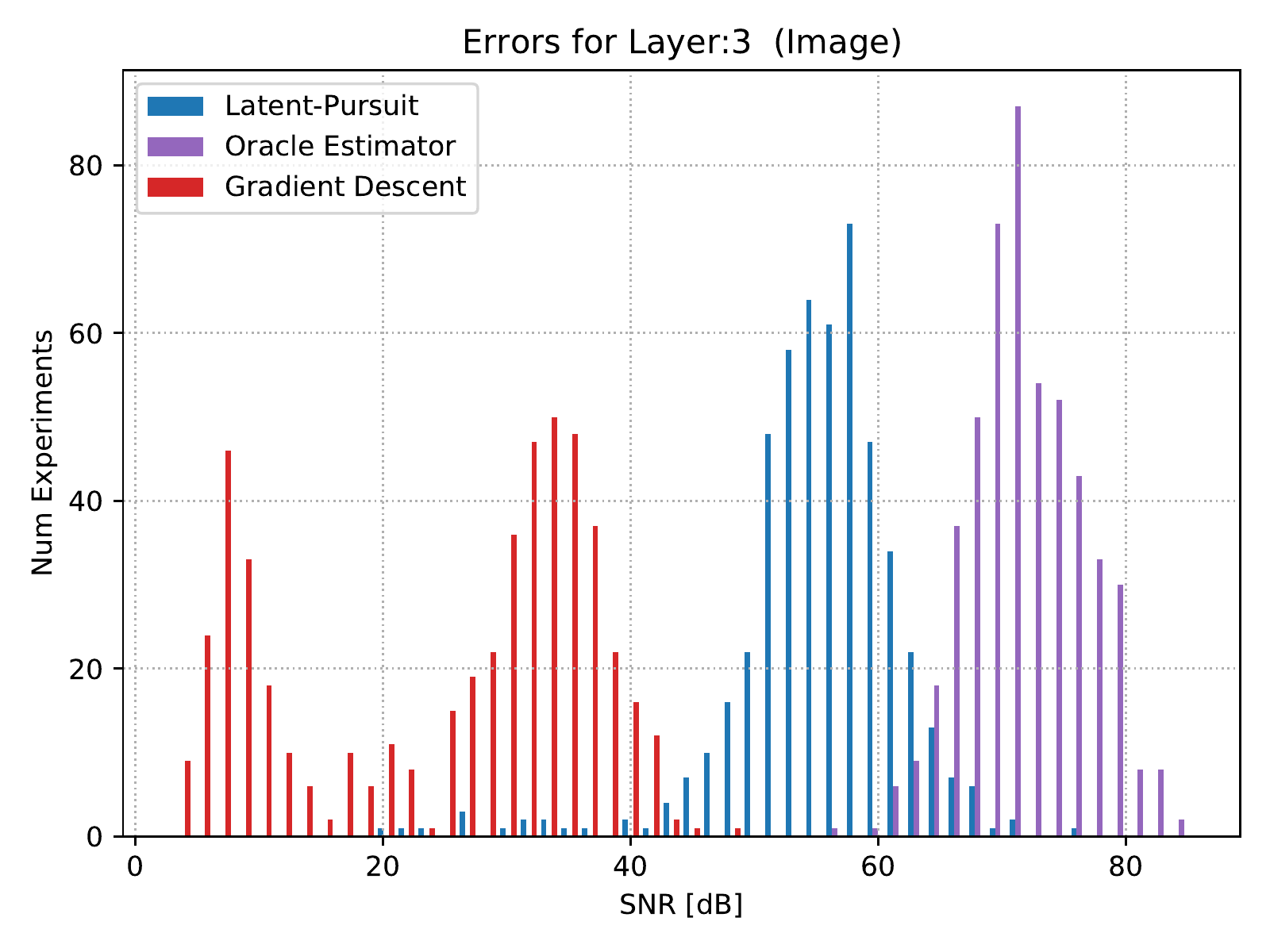}
    \caption{Image  $G(\rvz)$}
\end{subfigure}
\caption{Random mask inpainting reconstruction error for all the layers on a trained generator.}
\label{fig:inpainting_random_mse}
\end{figure}

\begin{figure}[H]
    \centering
    \begin{subfigure}{0.25\textwidth}
        Ground truth \vspace*{10pt} \\ Masked input \vspace*{10pt} \\ \vspace*{10pt}Gradient descent \\ Our approach 
    \end{subfigure}
    \begin{subfigure}{0.7\textwidth}
        \includegraphics[trim={10 80 10 80},clip,width=1\linewidth]{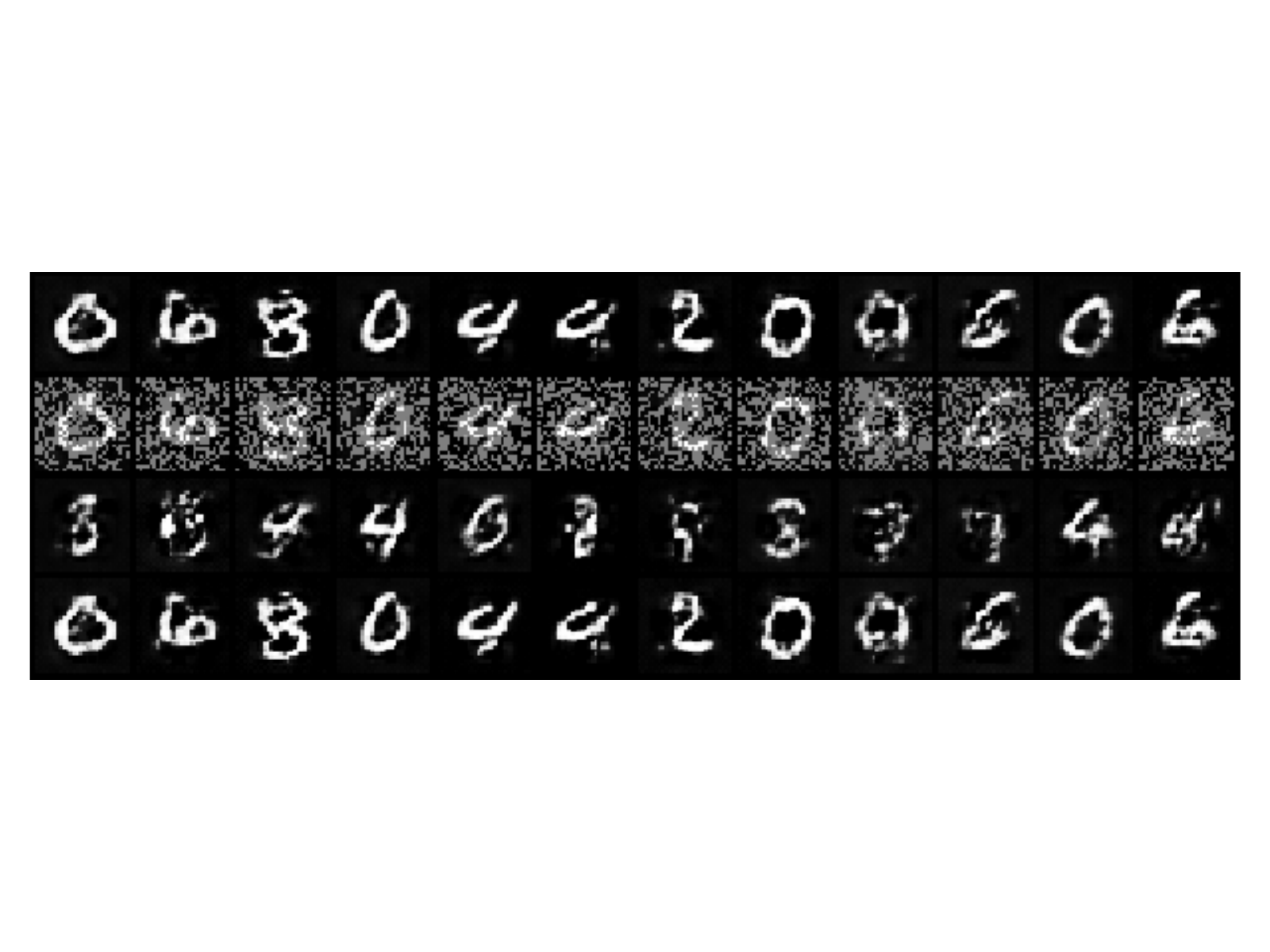}
    \end{subfigure}
    \caption{Random mask inpainting gradient descent failed reconstructions.}
    \label{fig:inpainting_random_images_fail}
\end{figure}

\begin{figure}[H]
    \centering
    \begin{subfigure}{0.25\textwidth}
        Ground truth \vspace*{10pt} \\ Masked input \vspace*{10pt} \\ \vspace*{10pt}Gradient descent \\ Our approach 
    \end{subfigure}
    \begin{subfigure}{0.7\textwidth}
        \includegraphics[trim={10 80 10 80},clip,width=1\linewidth]{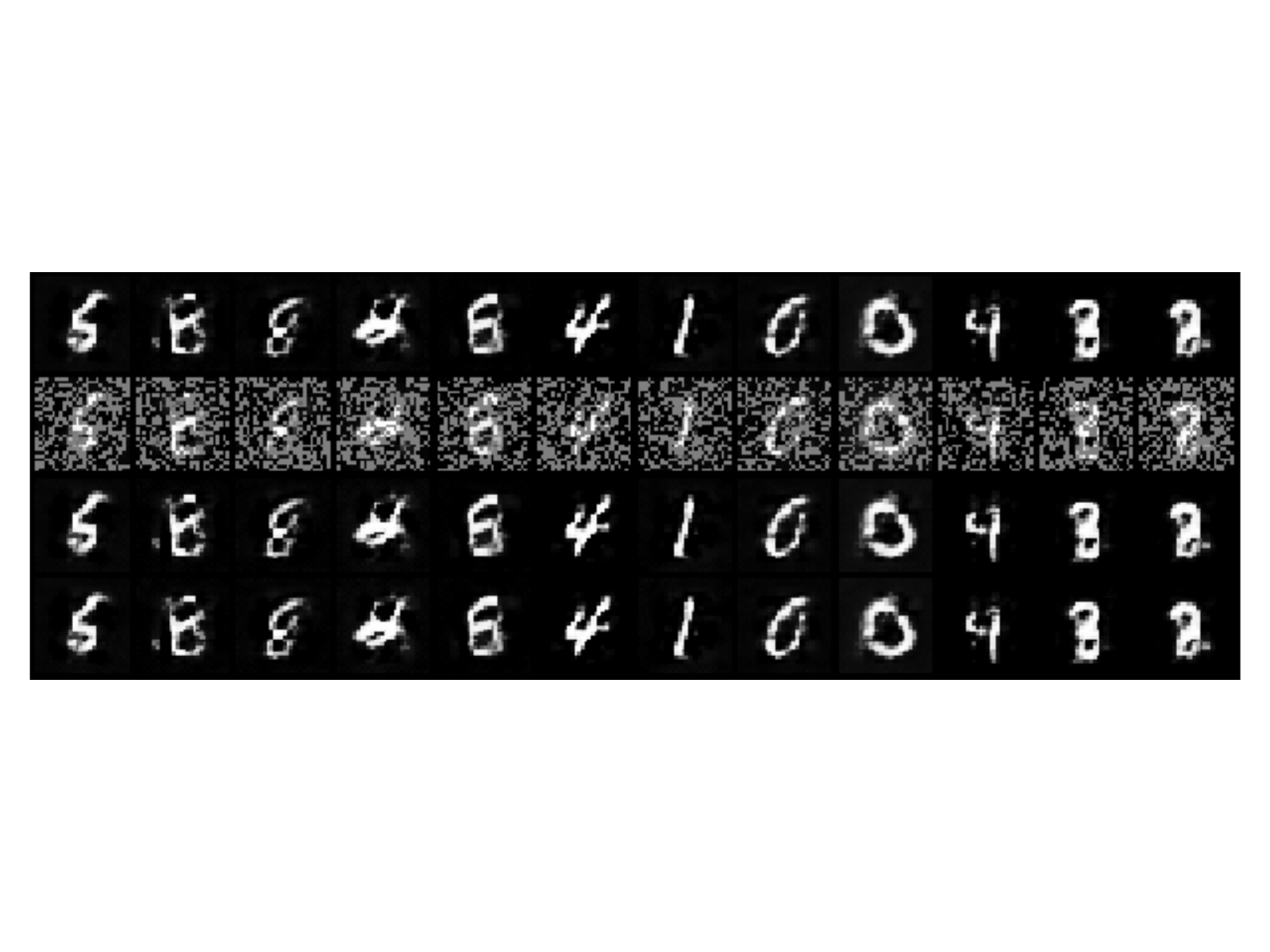}
    \end{subfigure}
    \caption{Random mask inpainting gradient descent successful reconstructions.}
    \label{fig:inpainting_random_images_succ}
\end{figure}

We repeat the above experiment with a similar setting, only this time the mask conceals the upper $\sim 45\%$ of each image instead of acting randomly (13 out of 28 rows). The results of this experiment, which are provided in Figures \ref{fig:inpainting_top_mse}-\ref{fig:inpainting_top_images_succ}, lead to similar conclusions as in the previous experiment. Note that since the model contains fully connected layers, we expect the two experiments to show similar results (as opposed to a convolutional model).

\begin{figure}[H]
\centering
\begin{subfigure}{0.49\textwidth}
\centering
    \includegraphics[width=1\linewidth]{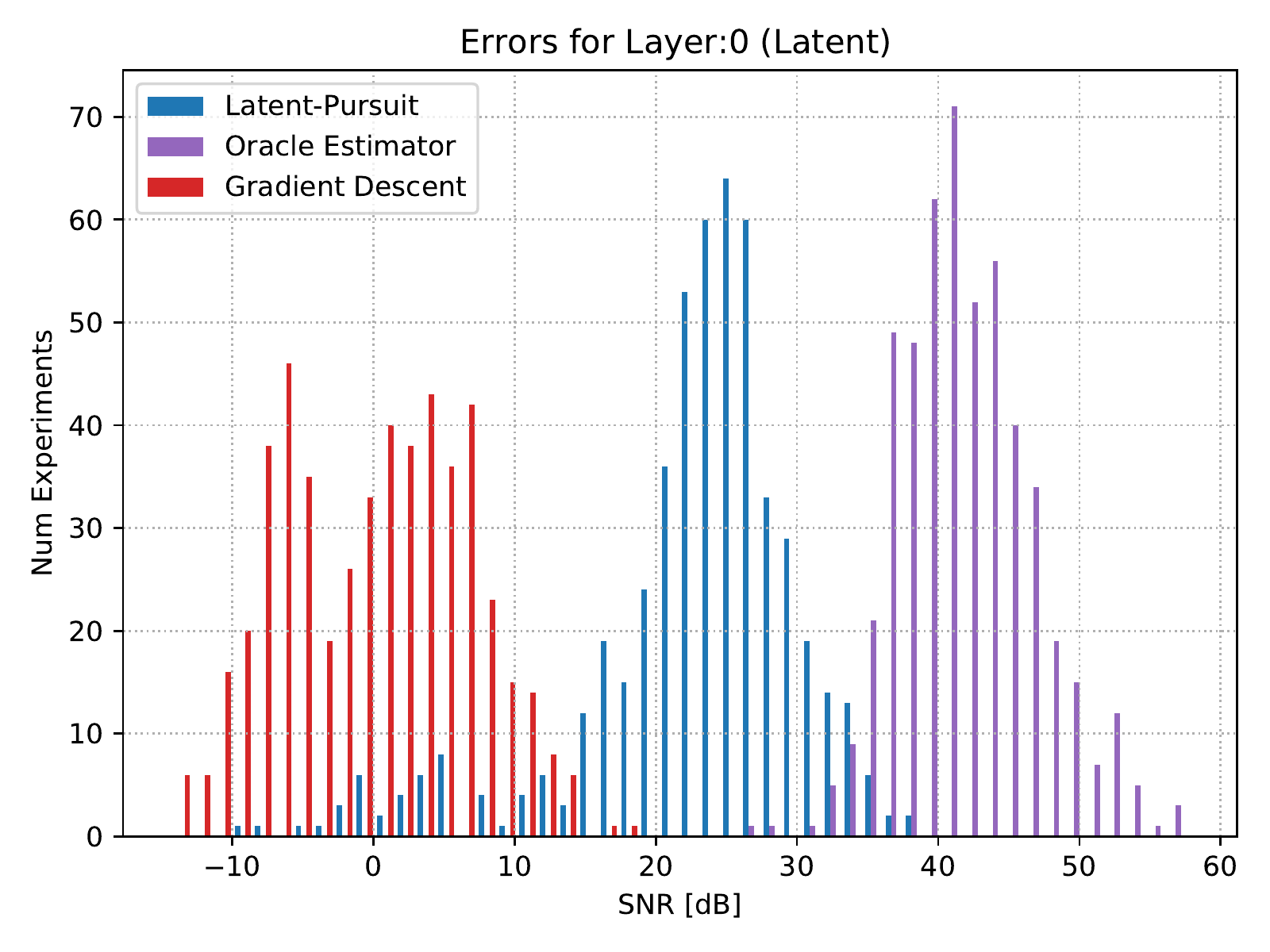}
    \caption{Latent vector $\rvz$}
\end{subfigure}
\begin{subfigure}{0.49\textwidth}
\centering
    \includegraphics[width=1\linewidth]{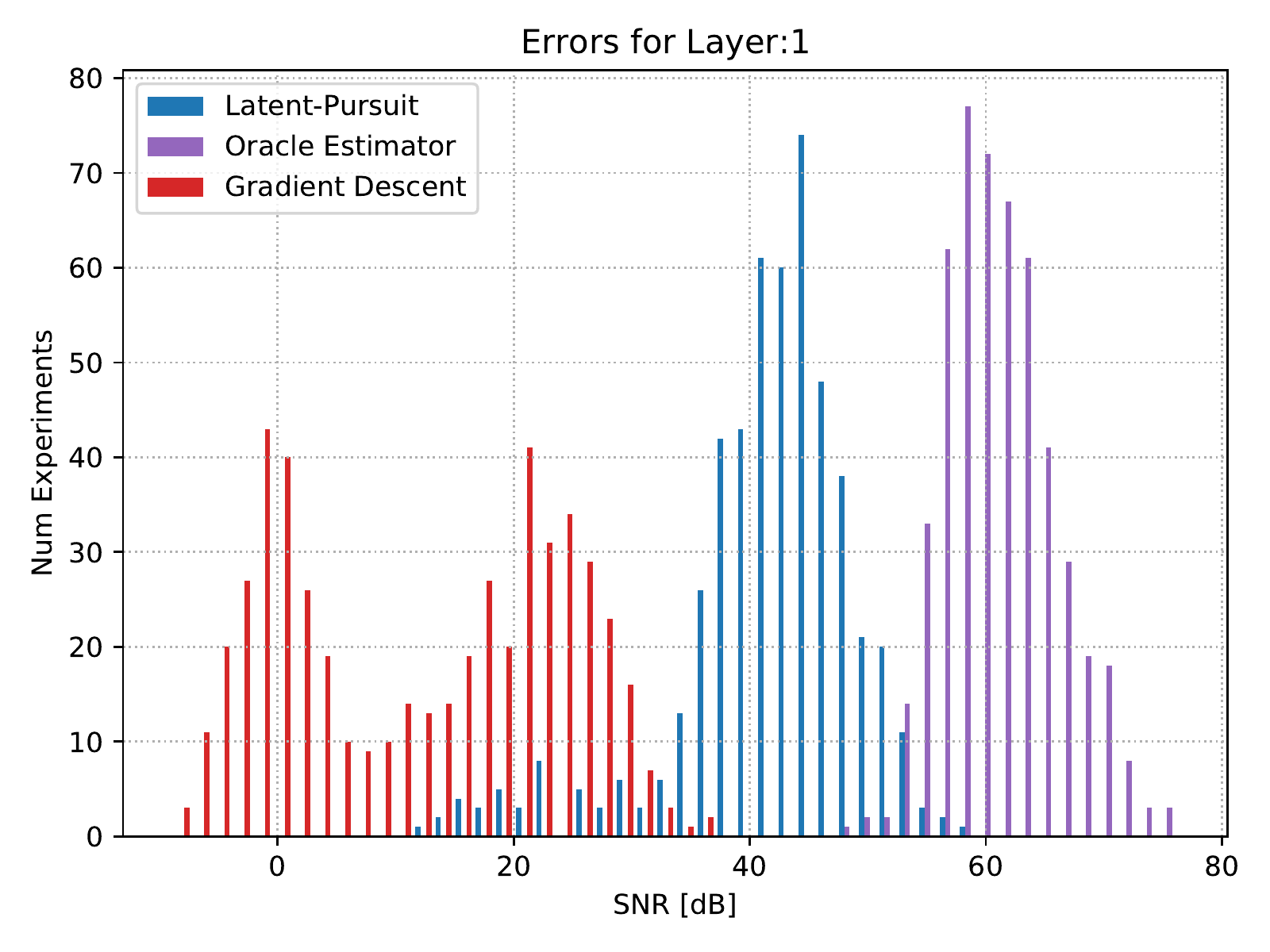}
    \caption{First hidden layer $\rvx_1$}
\end{subfigure}
\\
\begin{subfigure}{0.49\textwidth}
\centering
    \includegraphics[width=1\linewidth]{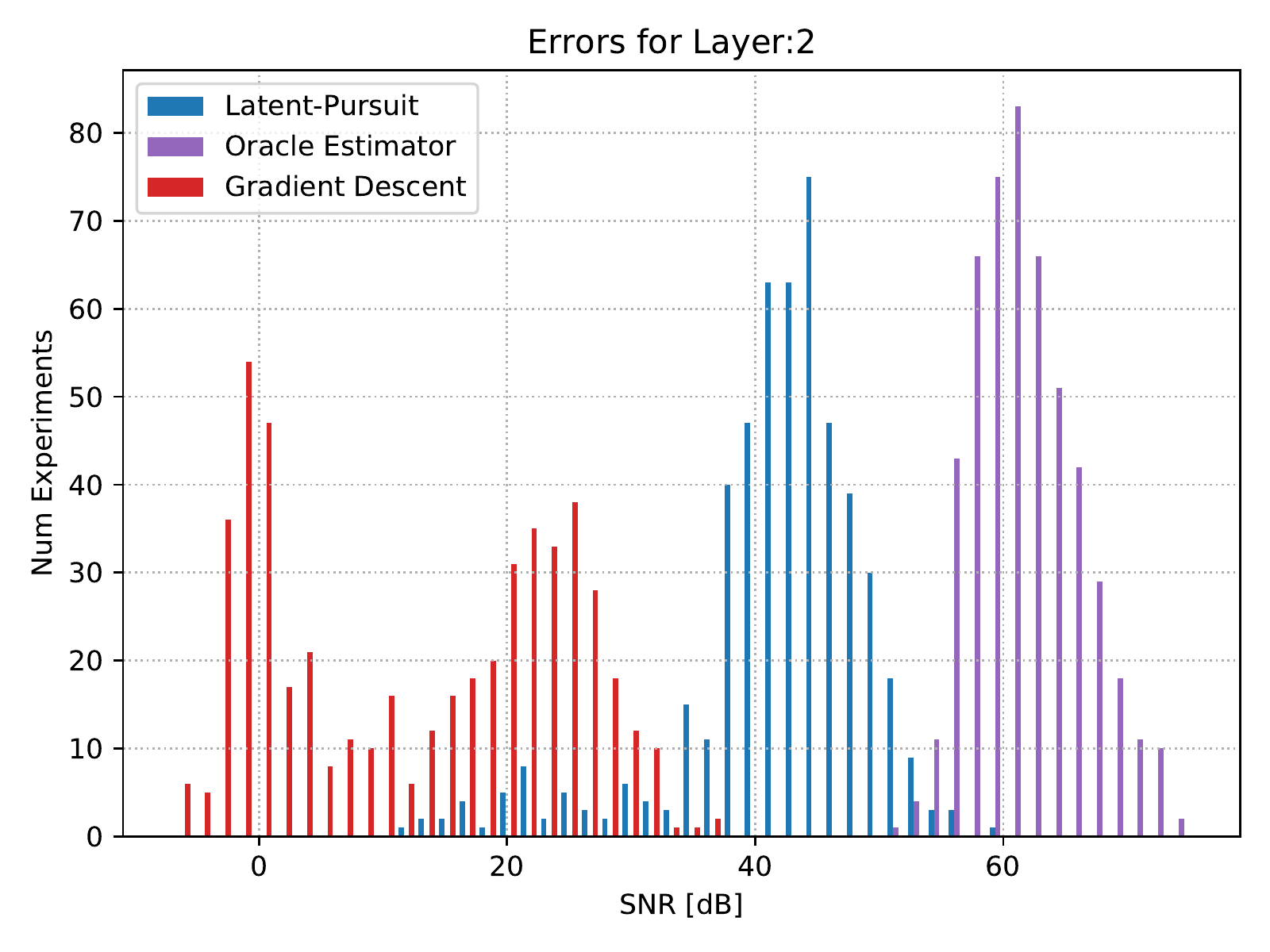}
    \caption{Second hidden layer  $\rvx_2$}
\end{subfigure}
\begin{subfigure}{0.49\textwidth}
\centering
    \includegraphics[width=1\linewidth]{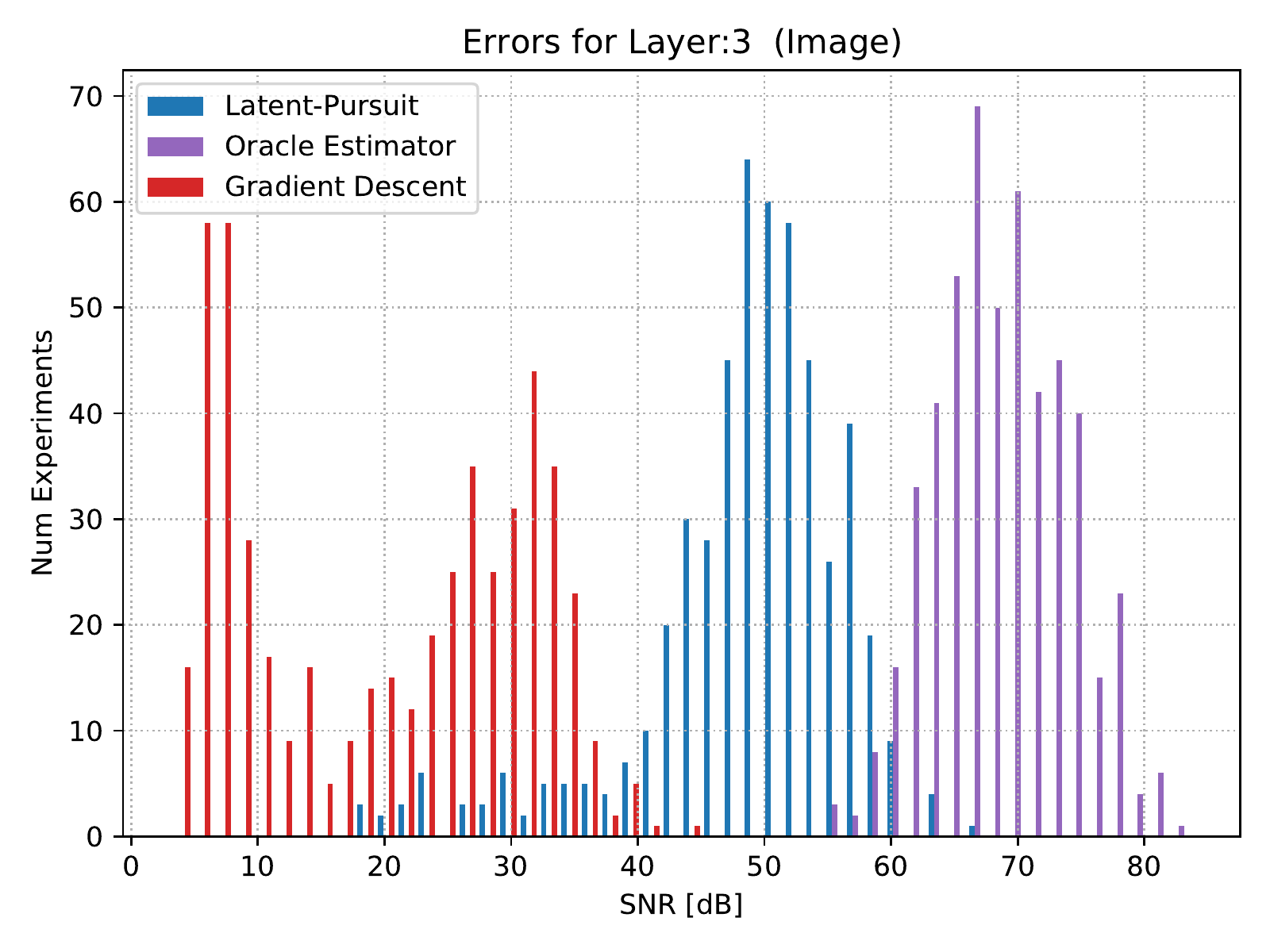}
    \caption{Image  $G(\rvz)$}
\end{subfigure}
\caption{Half image mask inpainting reconstruction error for all the layers.}
\label{fig:inpainting_top_mse}
\end{figure}

\begin{figure}[H]
    \centering
    \begin{subfigure}{0.25\textwidth}
        Ground truth \vspace*{10pt} \\ Masked input \vspace*{10pt} \\ \vspace*{10pt}Gradient descent \\ Our approach 
    \end{subfigure}
    \begin{subfigure}{0.7\textwidth}
        \includegraphics[trim={10 80 10 80},clip,width=1\linewidth]{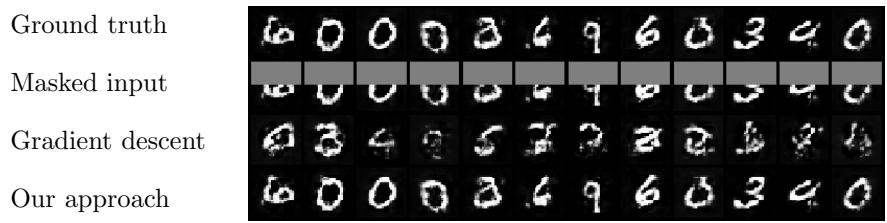}
    \end{subfigure}
    \caption{Half image mask inpainting gradient descent failed reconstructions.}
    \label{fig:inpainting_top_images_fail}
\end{figure}

\begin{figure}[H]
    \centering
    \begin{subfigure}{0.25\textwidth}
        Ground truth \vspace*{10pt} \\ Masked input \vspace*{10pt} \\ \vspace*{10pt}Gradient descent \\ Our approach 
    \end{subfigure}
    \begin{subfigure}{0.7\textwidth}
        \includegraphics[trim={10 80 10 80},clip,width=1\linewidth]{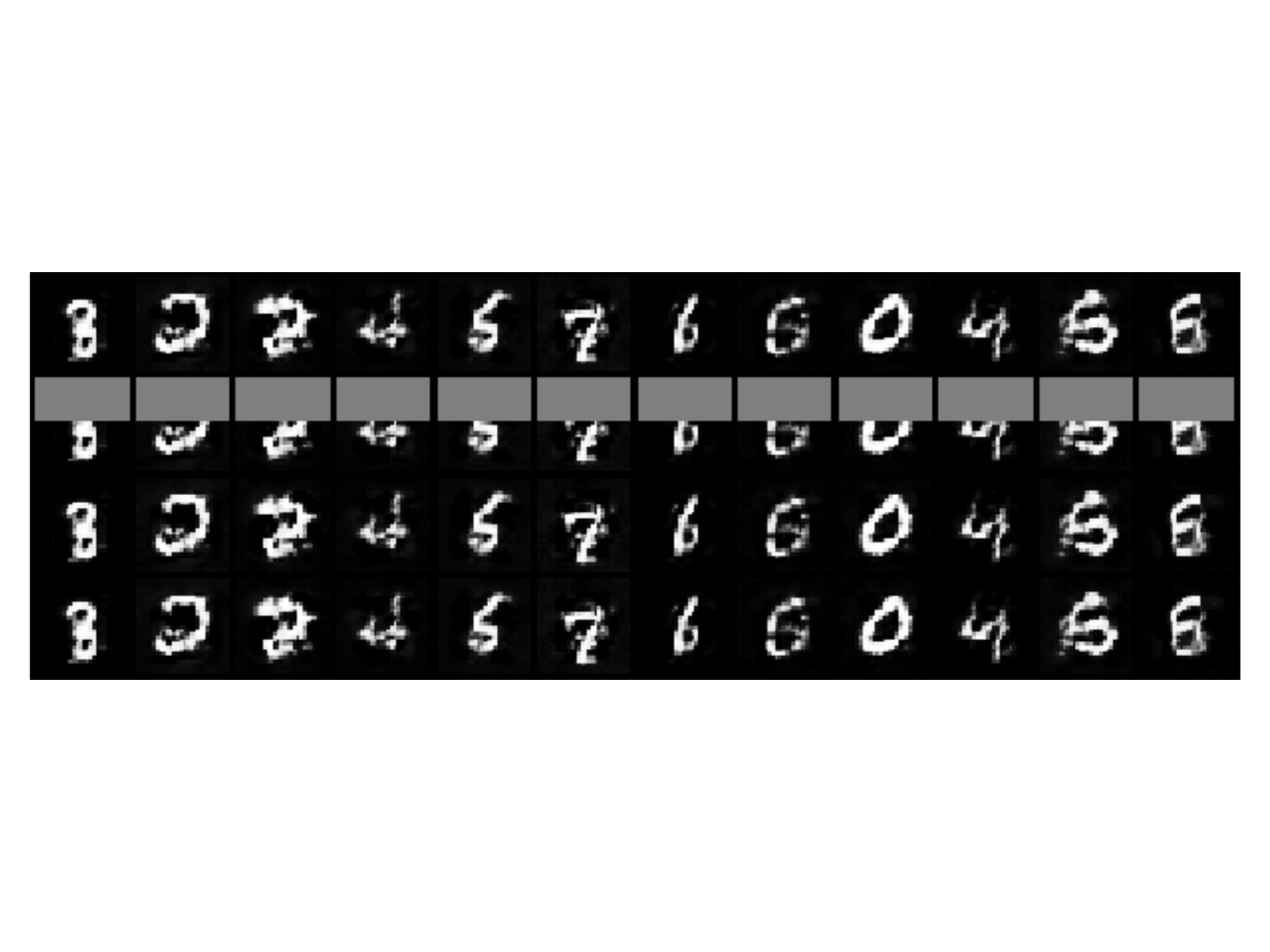}
    \end{subfigure}
    \caption{Half image mask inpainting gradient descent successful reconstructions.}
    \label{fig:inpainting_top_images_succ}
\end{figure}

\section{Conclusions}
In this paper we have introduced a novel perspective regarding the inversion of deep generative networks and its connection to sparse representation theory. Building on this, we have proposed novel invertibility guarantees for such a model for both random and trained networks. We have accompanied our analysis by novel pursuit algorithms for this inversion and presented numerical experiments that validate our theoretical claims and the superiority of our approach compared to the more classic gradient descent. We believe that the insights underlining this work could lead to a broader activity which further improves the inversion of these models in a variety of signal processing tasks.

\FloatBarrier

\bibliography{bib}

\begin{thebibliography}{10}

\bibitem{aberdam2019multi}
Aviad Aberdam, Jeremias Sulam, and Michael Elad.
\newblock Multi-layer sparse coding: the holistic way.
\newblock {\em SIAM Journal on Mathematics of Data Science}, 1(1):46--77, 2019.

\bibitem{bau2019seeing}
David Bau, Jun-Yan Zhu, Jonas Wulff, William Peebles, Hendrik Strobelt, Bolei
  Zhou, and Antonio Torralba.
\newblock Seeing what a gan cannot generate.
\newblock In {\em Proceedings of the IEEE International Conference on Computer
  Vision}, pages 4502--4511, 2019.

\bibitem{beck2017first}
Amir Beck.
\newblock {\em First-order methods in optimization}, volume~25.
\newblock SIAM, 2017.

\bibitem{beck2009fast}
Amir Beck and Marc Teboulle.
\newblock A fast iterative shrinkage-thresholding algorithm for linear inverse
  problems.
\newblock {\em SIAM journal on imaging sciences}, 2(1):183--202, 2009.

\bibitem{bora2017compressed}
Ashish Bora, Ajil Jalal, Eric Price, and Alexandros~G Dimakis.
\newblock Compressed sensing using generative models.
\newblock In {\em International Conference on Machine Learning (ICML)}, pages
  537--546. JMLR. org, 2017.

\bibitem{boyd2011distributed}
Stephen Boyd, Neal Parikh, Eric Chu, Borja Peleato, and Jonathan Eckstein.
\newblock Distributed optimization and statistical learning via the alternating
  direction method of multipliers.
\newblock {\em Foundations and Trends{\textregistered} in Machine learning},
  3(1):1--122, 2011.

\bibitem{chun2019convolutional}
Il~Yong Chun and Jeffrey~A Fessler.
\newblock Convolutional analysis operator learning: acceleration and
  convergence.
\newblock {\em IEEE Transactions on Image Processing}, 29(1):2108--2122, 2019.

\bibitem{donahue2016adversarial}
Jeff Donahue, Philipp Kr{\"a}henb{\"u}hl, and Trevor Darrell.
\newblock Adversarial feature learning.
\newblock {\em arXiv preprint arXiv:1605.09782}, 2016.

\bibitem{donoho2003optimally}
David~L Donoho and Michael Elad.
\newblock Optimally sparse representation in general (nonorthogonal)
  dictionaries via $\ell_1$ minimization.
\newblock {\em Proceedings of the National Academy of Sciences},
  100(5):2197--2202, 2003.

\bibitem{elad2010sparse}
Michael Elad.
\newblock {\em Sparse and redundant representations: from theory to
  applications in signal and image processing}.
\newblock Springer Science \& Business Media, 2010.

\bibitem{goodfellow2014generative}
Ian Goodfellow, Jean Pouget-Abadie, Mehdi Mirza, Bing Xu, David Warde-Farley,
  Sherjil Ozair, Aaron Courville, and Yoshua Bengio.
\newblock Generative adversarial nets.
\newblock In {\em Advances in neural information processing systems}, pages
  2672--2680, 2014.

\bibitem{hand2018phase}
Paul Hand, Oscar Leong, and Vlad Voroninski.
\newblock Phase retrieval under a generative prior.
\newblock In {\em Advances in Neural Information Processing Systems}, pages
  9136--9146, 2018.

\bibitem{hand2019global}
Paul Hand and Vladislav Voroninski.
\newblock Global guarantees for enforcing deep generative priors by empirical
  risk.
\newblock {\em IEEE Transactions on Information Theory}, 66(1):401--418, 2019.

\bibitem{huang2018provably}
Wen Huang, Paul Hand, Reinhard Heckel, and Vladislav Voroninski.
\newblock A provably convergent scheme for compressive sensing under random
  generative priors.
\newblock {\em arXiv preprint arXiv:1812.04176}, 2018.

\bibitem{kingma2013auto}
Diederik~P Kingma and Max Welling.
\newblock Auto-encoding variational bayes.
\newblock {\em arXiv preprint arXiv:1312.6114}, 2013.

\bibitem{latorre2019fast}
Fabian Latorre, Armin eftekhari, and Volkan Cevher.
\newblock Fast and provable admm for learning with generative priors.
\newblock In {\em Advances in Neural Information Processing Systems}, pages
  12004--12016, 2019.

\bibitem{lei2019inverting}
Qi~Lei, Ajil Jalal, Inderjit~S Dhillon, and Alexandros~G Dimakis.
\newblock Inverting deep generative models, one layer at a time.
\newblock In {\em Advances in Neural Information Processing Systems}, pages
  13910--13919, 2019.

\bibitem{miyato2018spectral}
Takeru Miyato, Toshiki Kataoka, Masanori Koyama, and Yuichi Yoshida.
\newblock Spectral normalization for generative adversarial networks.
\newblock {\em arXiv preprint arXiv:1802.05957}, 2018.

\bibitem{nicolae2018plu}
Andrei Nicolae.
\newblock Plu: The piecewise linear unit activation function.
\newblock {\em arXiv preprint arXiv:1809.09534}, 2018.

\bibitem{papyan2017convolutional}
Vardan Papyan, Yaniv Romano, and Michael Elad.
\newblock Convolutional neural networks analyzed via convolutional sparse
  coding.
\newblock {\em The Journal of Machine Learning Research}, 18(1):2887--2938,
  2017.

\bibitem{romano2019adversarial}
Yaniv Romano, Aviad Aberdam, Jeremias Sulam, and Michael Elad.
\newblock Adversarial noise attacks of deep learning architectures: Stability
  analysis via sparse-modeled signals.
\newblock {\em Journal of Mathematical Imaging and Vision}, pages 1--15, 2019.

\bibitem{shah2018solving}
Viraj Shah and Chinmay Hegde.
\newblock Solving linear inverse problems using gan priors: An algorithm with
  provable guarantees.
\newblock In {\em IEEE International Conference on Acoustics, Speech and Signal
  Processing (ICASSP)}, pages 4609--4613. IEEE, 2018.

\bibitem{simon2020barycenters}
Dror Simon and Aviad Aberdam.
\newblock Barycenters of natural images constrained wasserstein barycenters for
  image morphing.
\newblock In {\em Proceedings of the IEEE/CVF Conference on Computer Vision and
  Pattern Recognition}, pages 7910--7919, 2020.

\bibitem{sulam2019multi}
Jeremias Sulam, Aviad Aberdam, Amir Beck, and Michael Elad.
\newblock On multi-layer basis pursuit, efficient algorithms and convolutional
  neural networks.
\newblock {\em IEEE transactions on pattern analysis and machine intelligence},
  2019.

\bibitem{sulam2018multilayer}
Jeremias Sulam, Vardan Papyan, Yaniv Romano, and Michael Elad.
\newblock Multilayer convolutional sparse modeling: Pursuit and dictionary
  learning.
\newblock {\em IEEE Transactions on Signal Processing}, 66(15):4090--4104,
  2018.

\bibitem{tropp2006just}
Joel~A Tropp.
\newblock Just relax: Convex programming methods for identifying sparse signals
  in noise.
\newblock {\em IEEE transactions on information theory}, 52(3):1030--1051,
  2006.

\bibitem{xin2016maximal}
Bo~Xin, Yizhou Wang, Wen Gao, David Wipf, and Baoyuan Wang.
\newblock Maximal sparsity with deep networks?
\newblock In {\em Advances in Neural Information Processing Systems (NeurIPS)},
  pages 4340--4348, 2016.

\bibitem{zhu2016generative}
Jun-Yan Zhu, Philipp Kr{\"a}henb{\"u}hl, Eli Shechtman, and Alexei~A Efros.
\newblock Generative visual manipulation on the natural image manifold.
\newblock In {\em European Conference on Computer Vision}, pages 597--613.
  Springer, 2016.

\end{thebibliography}
\bibliographystyle{plain}

% -----------------------------------------------------------------------------
% -----------------------------------------------------------------------------
% -----------------------------------------------------------------------------

\appendix

\section{Theorem \ref{thm:uniqueness}: Proof}
\label{app:uniqueness}

\begin{proof}

The main idea of the proof is to show that under the conditions of Theorem \ref{thm:uniqueness} the inversion task at every layer $i \in \{1,\ldots,L+1\}$ has a unique global minimum. For this goal we utilize the well-known uniqueness guarantee from sparse representation theory.
\begin{lemma}[Sparse Representation - Uniqueness Guarantee \cite{donoho2003optimally,elad2010sparse}]
    If a system of linear equations $\rvy = \rmW \rvx$ has a solution $\rvx$ satisfying $\norm{\rvx}_0 < \spark(\rmW)/2$, then this solution is necessarily the sparset possible.
\end{lemma}

Using the above Lemma, we can conclude that if $\rvx_L$ obeys $\norm{\rvx_L}_0 = s_L < \spark(\rmW_L)/2$, then $\rvx_L$ is the \emph{unique} vector that has at most $s_L$ nonzeros, while satisfying the equation $\phi^{-1}(\rvx) = \rmW_L \rvx_L$. 

Moving on to the previous layer, we can employ again the above Lemma for the supported vector $\rvx_L^{\S_L}$. This way, we can ensure that $\rvx_{L-1}$ is the unique $s_{L-1}$-sparse solution of $\rvx_L^{\S_L} = \rmW_{L-1}^{\S_L} \rvx_{L-1}$ as long as
\begin{equation}
    s_{L-1} = \norm{\rvx_{L-1}}_0 < \frac{\spark(\rmW_{L-1}^{\S_L})}{2}.
\end{equation}
However, the condition $s_{L-1} = \norm{\rvx_{L-1}}_0 < \frac{\subspark(\rmW_{L-1}, s_L)}{2}$ implies that the above necessarily holds. This way we can ensure that each layer $i$, $i\in \{1,\ldots,L-1\}$ is the unique sparse solution.

Finally, in order to invert the first layer we need to solve $\rvx_1^{\S_1} = \rmW_0^{\S_1} \rvz$. If $\rmW_0^{\S_1}$ has full column-rank, this system either has no solution or a unique one. In our case, we do know that a solution exists, and thus, necessarily, it is unique. A necessary but insufficient condition for this to be true is $s_1 \ge n_0$. The additional requirement $\subrank(\rmW_0, s_1) = n_0 \le s_1$ is sufficient for $\rvz$ to be the unique solution, and this concludes the proof.

\end{proof}

\section{The Oracle Estimator}
\label{app:oracle}

The motivation for studying the recovery ability of the Oracle is that it can reveal the power of utilizing the inherent sparsity of the feature maps. Therefore, we analyze the layer-wise Oracle estimator described in Algorithm \ref{alg:layered_oracle}, which is similar to the layer-by-layer fashion we adopt in both the Layered Basis-Pursuit (Algorithm \ref{alg:layered_bp}) and in the Latent-Pursuit (Algorithm \ref{alg:latent_pursuit}). In this analysis we assume that the contaminating noise is white additive Gaussian. 

\begin{algorithm}[H]
\caption{The Layered-Wise Oracle} \label{alg:layered_oracle}
\textbf{Input:} $\rvy = G(\rvz) + \rve \in \R^n$, and \emph{supports of each layer} $\{\S_i\}_{i=1}^L$.\\
\textbf{First step:} $\hrvx_L = \argmin_\rvx ~ \frac{1}{2}\norm{\phi^{-1}(\rvy) - \brmW_L \rvx}_2^2$, where $\brmW_L$ is the column supported matrix $\rmW_L[:, \S_L]$.\\
\textbf{Intermediate steps:} For any layer $i =L-1,\ldots,1$, set $\hrvx_i = \argmin_\rvx ~ \frac{1}{2}\norm{\hrvx_{i+1}^{\S_{i+1}} - \brmW_i \rvx}_2^2$, where $\brmW_i$ is the row and column supported matrix $\rmW_i[\S_{i+1}, \S_i]$.\\
\textbf{Final step:} Set $\hrvz = \argmin_\rvz ~ \frac{1}{2}\norm{\hrvx_1^{\S_1} - \rmW_0^{\S_1} \rvz}_2^2$.
\end{algorithm}

The noisy signal $\rvy$ carries an additive noise with energy proportional to its dimension, $\sigma^2 n$. Theorem \ref{thm:oracle} below suggests that the Oracle can attenuate this noise by a factor of $\frac{n_0}{n}$, which is typically much smaller than $1$. Moreover, the error in each layer is proportional to its cardinality $\sigma^2 s_i$. These results are expected, as the Oracle simply projects the noisy signal on low-dimensional subspaces of known dimension. That said, this result reveals another advantage of employing the sparse coding approach over solving least squares problems, as the error can be proportional to $s_i$ rather than to $n_i$. 

\begin{theorem}[The Oracle] \label{thm:oracle}
    Given a noisy signal $\rvy = G(\rvz) + \rve$, where $\rve \sim \N(\rvzero,\sigma^2\rmI)$, and assuming known supports $\{\S_i\}_{i=1}^L$, the recovery errors satisfy \footnote{For simplicity we assume here that $\phi$ is the identity function.}:
    \begin{equation}
        \frac{\sigma^2}{\prod_{j=i}^L \lmax(\brmW_j^T \brmW_j)} s_i \leq \E \norm{\hrvx_i - \rvx_i}_2^2 \leq \frac{\sigma^2}{\prod_{j=i}^L \lmin(\brmW_j^T \brmW_j)} s_i,
    \end{equation}
    for $i \in \{1, \ldots, L\}$, where $\brmW_i$ is the row and column supported matrix, $\rmW_i[\S_{i+1},\S_i]$. The recovery error bounds for the latent vector are similarly given by:
    \begin{equation}
        \frac{\sigma^2}{\prod_{j=0}^L \lmax(\brmW_j^T \brmW_j)} n_0 \leq \E \norm{\hrvz - \rvz}_2^2 \leq \frac{\sigma^2}{\prod_{j=0}^L \lmin(\brmW_j^T \brmW_j)} n_0.
    \end{equation}
\end{theorem}

\begin{proof}
Assume $\rvy = \rvx + \rve$ with $\rvx = G(\rvz)$, then the Oracle for the $L$th layer is $\hrvx_L^{\S} = \brmW_L^\dagger \rvy$. Since $\rvy = \brmW_L \rvx_L^S + \rve$, we get that $\hrvx_L^{\S} = \rvx_L^{\S} + \trve_L$, where $\trve_L = \brmW_L^\dagger \rve$, and $\trve_L \sim \N(\rvzero, \sigma^2 (\brmW_L^T \brmW_L)^{-1})$. Therefore, using the same proof technique as in \cite{aberdam2019multi}, the upper bound on the recovery error in the $L$th layer is:
\begin{equation}
    \E \norm{\hrvx_L - \rvx_L}_2^2 = \sigma^2 \trace ((\brmW_L^T \brmW_L)^{-1}) \leq \sigma^2 \frac{s_L}{ \lmin(\brmW_L^T \brmW_L)}.
\end{equation}
Using the same approach we can derive the lower bound by using the largest eigenvalue of $\brmW_L^T \brmW_L$. 

In a similar fashion, we can write $\hrvx_i^{\S} = \rvx_i^{\S} + \trve_i$ for all $i \in \{0, \ldots, L-1\}$, where $\trve_i = \rmA_{[i,L]} \rve$ and $\rmA_{[i,L]} \triangleq \brmW_i^\dagger \brmW_{i+1}^\dagger \cdots \brmW_L^\dagger $. Therefore, the upper bound for the recovery error in the $i$th layer becomes:
\begin{equation}
\begin{split}
    \E \norm{\hrvx_i - \rvx_i}_2^2 & = \E \norm{\rmA_{[i,L]} \rve}_2^2 \\
    & = \sigma^2\trace \left( \rmA_{[i,L]} \rmA_{[i,L]}^T \right) \\
    & = \sigma^2\trace \left( \rmA_{[i,L-1]} \brmW_L^\dagger (\brmW_L^\dagger)^T  \rmA_{[i,L-1]}^T \right) \\
    & = \sigma^2\trace \left( \rmA_{[i,L-1]} (\brmW_L^T \brmW_L)^{-1}  \rmA_{[i,L-1]}^T \right) \\
    & \leq \frac{\sigma^2}{\lmin(\brmW_L^T \brmW_L)} \trace \left( \rmA_{[i,L-1]} \rmA_{[i,L-1]}^T \right) \\
    & \leq ~~ \cdots ~~ \\
    & \leq \frac{\sigma^2}{\prod_{j=i+1}^L \lmin(\brmW_j^T \brmW_j)}\trace \left( \rmA_{[i,i]} \rmA_{[i,i]}^T \right) \\
    & = \frac{\sigma^2}{\prod_{j=i+1}^L \lmin(\brmW_j^T \brmW_j)}\trace \left( (\brmW_i^T \brmW_i)^{-1} \right) \\
    & \leq \frac{\sigma^2}{\prod_{j=i}^L \lmin(\brmW_j^T \brmW_j)} s_i,
\end{split}
\end{equation}
and this concludes the proof.

\end{proof}

\section{Theorem \ref{thm:layered_bp}: Proof}
\label{app:layered_bp}

\begin{proof}
We first recall the stability guarantee from \cite{tropp2006just} for the basis-pursuit.

\begin{lemma}[Basis Pursuit Stability \cite{tropp2006just}] \label{lemma:bp}
Let $\rvx^*$ be an unknown sparse representation with known cardinality of $\norm{\rvx^*}_0 = s$, and let $\rvy = \rmW \rvx^* + \rve$, where $\rmW$ is a matrix with unit-norm columns and $\norm{\rve}_2 \leq \epsilon$. Assume the mutual coherence of the
dictionary $\rmW$ satisfies $s < 1/(3\mu(\rmW))$. Let $\hrvx = \argmin_\rvx \frac{1}{2}\norm{\rvy - \rmW \rvx}_2^2 + \lambda \norm{\rvx}_1$, with $\lambda = 2 \epsilon$. Then, $\hrvx$ is unique, the support of $\hrvx$ is a subset of the support of $\rvx^*$, and
\begin{equation}
    \norm{\rvx^* - \hrvx}_{\infty} < (3+\sqrt{1.5}) \epsilon.
\end{equation}
\end{lemma}

In order to use the above lemma in our analysis we need to modify it such that $\rmW$ does not need to be column normalized and that the error is $\ell_2$- and not $\ell_{\infty}$-bounded. For the first  modification we decompose a general unnormalized matrix $\rmW$ as $\trmW \rmD$, where $\trmW$ is the normalized matrix, $\trvw_i = \rvw_i / \norm{\rvw_i}_2$, and $\rmD$ is a diagonal matrix with $d_i = \norm{\rvw_i}_2$. Using the above lemma we get that 
\begin{equation}
    \norm{\rmD(\rvx^* - \hrvx)}_{\infty} < (3+\sqrt{1.5}) \epsilon.
\end{equation}
Thus, the error in $\hrvx$ is bounded by
%\begin{equation}
 %   \norm{\rvx^*_i - \hrvx_i}_{\infty} < %\frac{(3+\sqrt{1.5})}{\norm{\rvw_i}_2} \epsilon, \text{ for all } i \in \S.
%\end{equation}
%Taking a maximum over $i \in S$ results in that
\begin{equation}
    \norm{\rvx - \hrvx}_{\infty} < \frac{(3+\sqrt{1.5})}{\min_i \norm{\rvw_i}_2} \epsilon.
\end{equation}
Since Lemma \ref{lemma:bp} guarantees that the support of $\hrvx$ is a subset of the support of $\rvx^*$, we can conclude that 
\begin{equation}
    \norm{\rvx - \hrvx}_{2} < \frac{(3+\sqrt{1.5})}{\min_i \norm{\rvw_i}_2} \epsilon \sqrt{s}.
\end{equation}
Under the conditions of Theorem \ref{thm:layered_bp}, we can use the above conclusion to guarantee that estimating $\rvx_L$ from the noisy input $\rvy$ using Basis-Pursuit must lead to a unique $\hrvx_L$ such that its support is a subset of that of $\rvx_L$. Also, 
\begin{equation}
    \norm{\rvx_L - \hrvx_L}_{2} < \epsilon_L = \frac{(3+\sqrt{1.5})}{\min_j \norm{\rvw_{L,j}}_2} \epsilon_{L+1} \sqrt{s_L},
\end{equation}
where $\rvw_{L,j}$ is the $j$th column in $\rmW_L$, and $\epsilon_{L+1} = \epsilon \ell$ as $\phi^{-1}(\rvy)$ can increase the noise by a factor of $\ell$. 

Moving on to the estimation of the previous layer, we have that $\hrvx_L^{\hS_L} = \rmW_{L-1}^{\hS_L}\rvx_{L-1} + \rve_L$, where $\norm{\rve_L}_2 \leq \epsilon_L$. According to Theorem \ref{thm:layered_bp} assumptions, the mutual coherence condition holds, and therefore, we get that the support of $\hrvx_{L-1}$ is a subset of the support of $\rvx_{L-1}$,  $\hrvx_{L-1}$ is unique, and that
\begin{equation}
    \norm{\rvx_{L-1} - \hrvx_{L-1}}_{2} < \epsilon_{L-1} = \frac{(3+\sqrt{1.5})}{\min_j \norm{\rvw^{\hS_L}_{L-1,j}}_2} \epsilon_L \sqrt{s_{L-1}}.
\end{equation}
Using the same technique proof for all the hidden layers results in 
\begin{equation}
    \norm{\rvx_{i} - \hrvx_{i}}_{2} < \epsilon_{i} = \frac{(3+\sqrt{1.5})}{\min_j \norm{\rvw^{\hS_{i+1}}_{i,j}}_2} \epsilon_{i+1} \sqrt{s_{i}}, \text{ for all } i \in \{1, \ldots, L-1\},
\end{equation}
where $\rvw^{\hS_{i+1}}_{i,j}$ is the $j$th column in $\rmW_i^{\hS_{i+1}}$.

Finally, we have that $\hrvx_1^{\hS_1} = \rmW^{\hS_1}_0 \rvz + \rve_1$, where $\norm{\rve_1}_2 \leq \epsilon_1$. Therefore, if $\varphi = \lmin((\rmW_0^{\hS_1})^T \rmW_0^{\hS_1}) > 0$, and
\begin{equation}
    \hrvz = \argmin_\rvz \frac{1}{2} \norm{\hrvx_1^{\hS_1} - \rmW_0^{\hS_1} \rvz}_2^2.
\end{equation}
Then,
\begin{equation}
    \norm{\hrvz - \rvz}_2^2 = \rve_1^T \left( (\rmW_0^{\hS_1})^T \rmW_0^{\hS_1} \right)^{-1} \rve_1 \leq \frac{1}{\varphi} \epsilon_1^2,
\end{equation}
which concludes Theorem \ref{thm:layered_bp} guarantees.

\end{proof}

% \section{Layerwise Inversion Analysis for a Trained Network}
% \label{app:layerwise_inversion}
% Put some figures of the inversion process layer-by-layer\\
% \TBC

\end{document}